\def\year{2021}\relax

\documentclass[letterpaper]{article} 
\usepackage{aaai21}  
\usepackage{times}  
\usepackage{helvet} 
\usepackage{courier}  
\usepackage[hyphens]{url}  
\usepackage{graphicx} 
\usepackage{natbib}  
\usepackage{caption} 
\usepackage{url}            
\usepackage{booktabs}       
\usepackage{amsfonts}       
\usepackage{amsmath}
\usepackage{amsthm}
\usepackage{nicefrac}       
\usepackage{microtype}      
\usepackage{graphicx}
\usepackage{float}
\usepackage{subfigure}
\usepackage{mathtools}
\usepackage{multirow}

\usepackage{xcolor}

\newtheorem{proposition}{Proposition}

\newtheorem{remark}{Remark}

\title{Beyond Low-frequency Information in Graph Convolutional Networks}

\author{
	Deyu Bo\textsuperscript{\rm 1},
	Xiao Wang\textsuperscript{\rm 1},
	Chuan Shi\textsuperscript{\rm 1}	\thanks{Corresponding author: Chuan Shi (shichuan@bupt.edu.cn)},
	Huawei Shen\textsuperscript{\rm 2}
}

\affiliations{

    \textsuperscript{\rm 1}Beijing University of Posts and Telecommunications\\
    \textsuperscript{\rm 2}CAS Key Lab of Network Data Science and Technology, Institute of Computing Technology,\\
    	Chinese Academy of Sciences, Beijing, China\\
    \{bodeyu, xiaowang, shichuan\}@bupt.edu.cn, shenhuawei@ict.ac.cn
}

\begin{document}

\maketitle

\begin{abstract}
Graph neural networks (GNNs) have been proven to be effective in various network-related tasks. Most existing GNNs usually exploit the low-frequency signals of node features, which gives rise to one fundamental question: is the low-frequency information all we need in the real world applications? 
In this paper, we first present an experimental investigation assessing the roles of low-frequency and high-frequency signals, where the results clearly show that exploring low-frequency signal only is distant from learning an effective node representation in different scenarios. How can we adaptively learn more information beyond low-frequency information in GNNs? A well-informed answer can help GNNs enhance the adaptability. We tackle this challenge and propose a novel Frequency Adaptation Graph Convolutional Networks (FAGCN) with a self-gating mechanism, which can adaptively integrate different signals in the process of message passing. 
For a deeper understanding, we theoretically analyze the roles of low-frequency signals and high-frequency signals on learning node representations, which further explains why FAGCN can perform well on different types of networks.
Extensive experiments on six real-world networks validate that FAGCN not only alleviates the over-smoothing problem, but also has advantages over the state-of-the-arts.
\end{abstract}

\section{Introduction}
\label{introduction}

Networks, such as social networks, citation networks and molecular networks, are ubiquitous in the real world.
Recently, the emerging graph neural networks (GNNs) have demonstrated powerful ability to learn node representations by jointly encoding network structures and node features \cite{survey2, survey1, HINSurvey}.  This strategy has been proven to be effective in various tasks, including link prediction \cite{linkprediction}, node classification \cite{GCN, GAT} and graph classification \cite{graphclassification}.


In general, GNNs update node representations by aggregating information from neighbors, which can be seen as a special form of low-pass filter \cite{SGC, label2019li}. Some recent studies \cite{revisiting, GraphHeat} show that the smoothness of signals, i.e., low-frequency information, are the key to the success of GNNs.
However, is the low-frequency information all we need and what roles do other information play in GNNs? This is a fundamental question which motivates us to rethink whether GNNs comprehensively exploit the information in node features when learning node representation.

Firstly, the low-pass filter in GNNs mainly retains the commonality of node features, which inevitably ignores the difference, so that the learned representations of connected nodes become similar. Thanks to the smoothness of low-frequency information, this mechanism may work well for assortative networks, i.e., similar nodes tend to connect with each other \cite{GraphHeat}. 
However, the real-world networks are not always assortative, but sometimes disassortative, i.e., nodes from different classes tend to connect with each other \cite{mixing}. 
For example, the chemical interactions in proteins often occur between different types of amino acids~\cite{H2GNN}.
If we force the representation of connected proteins to be similar by employing low-pass filter, obviously, the performance will be largely hindered. The low-frequency information here is insufficient to support the inference in such networks. Under the circumstances, the high-frequency information, capturing the difference between nodes, may be more suitable. Even the raw features, containing both low- and high-frequency information, are alternative solution \cite{AM-GCN}. Secondly, it is well established that the node representation will becomes indistinguishable when we always utilize low-pass filter, causing over-smoothing \cite{lossexp}. This reminds us that low-pass filter of current GNNs is distant from optimal for real world scenarios.

To provide more evidence for the above analysis, we focus on low-frequency and high-frequency signals as an example, and present experiments to assess their roles (details can be seen in Section \ref{case}). 
The results clearly show that both of them are helpful for learning node representations. Specifically, we find that when a network exhibits disassortativity, high-frequency signals perform much better than low-frequency signals.
This implies that the high-frequency information, which is largely eliminated by the current GNNs, is not always useless, and the low-frequency information is not always optimal for the complex networks.
Once the weakness of low-frequency information in GNNs is identified, a natural question is \emph{how to use signals of different frequencies in GNNs and, at the same time, makes GNNs suitable for different type of networks?}

\begin{figure*}
\centering
\subfigure[Classification accuracy]{
\label{fig:vary_q}
\includegraphics[width=0.31\textwidth]{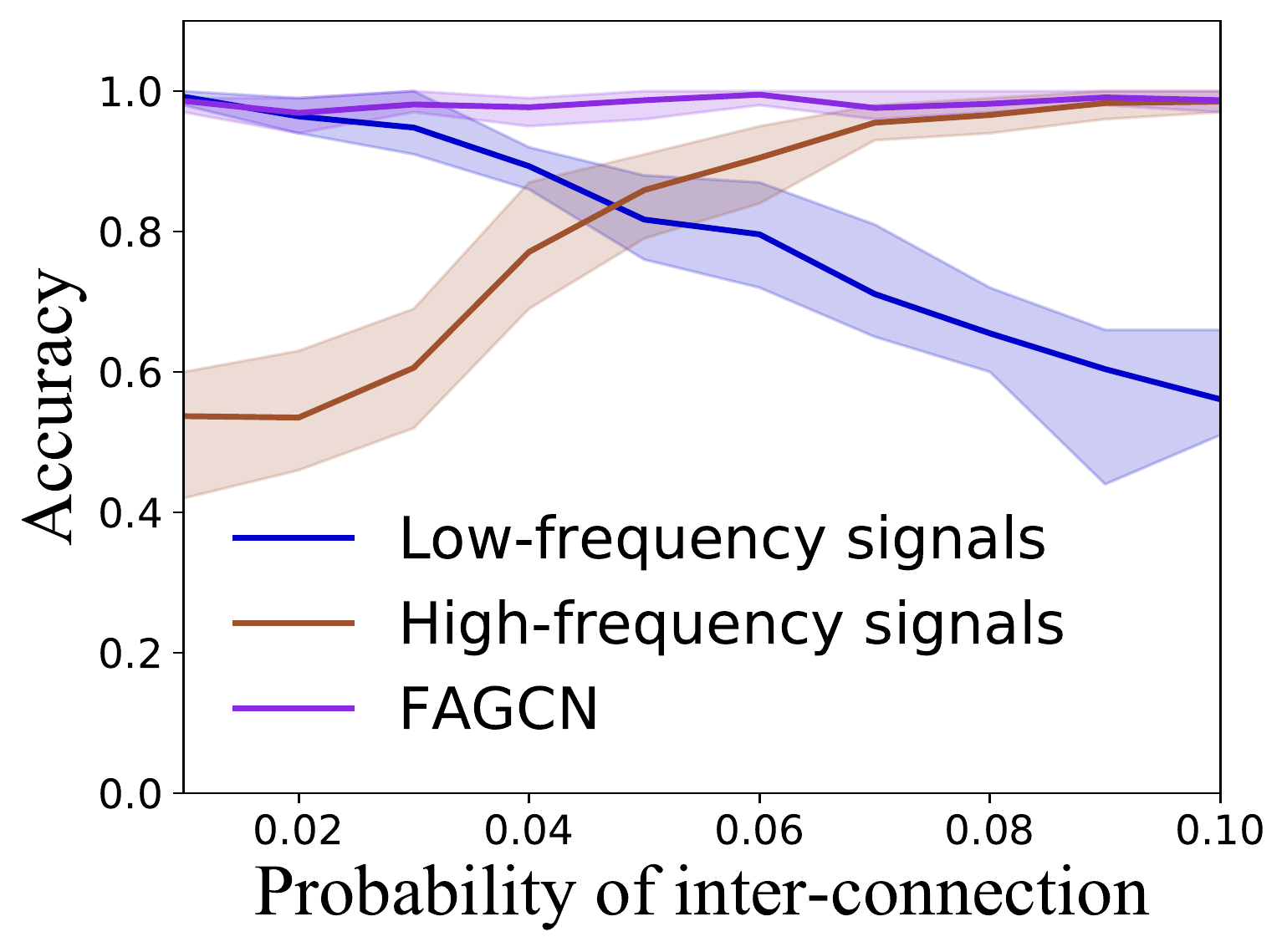}
}
\subfigure[Existing GNNs]{
\label{fig:GNNs}
\includegraphics[width=0.31\textwidth]{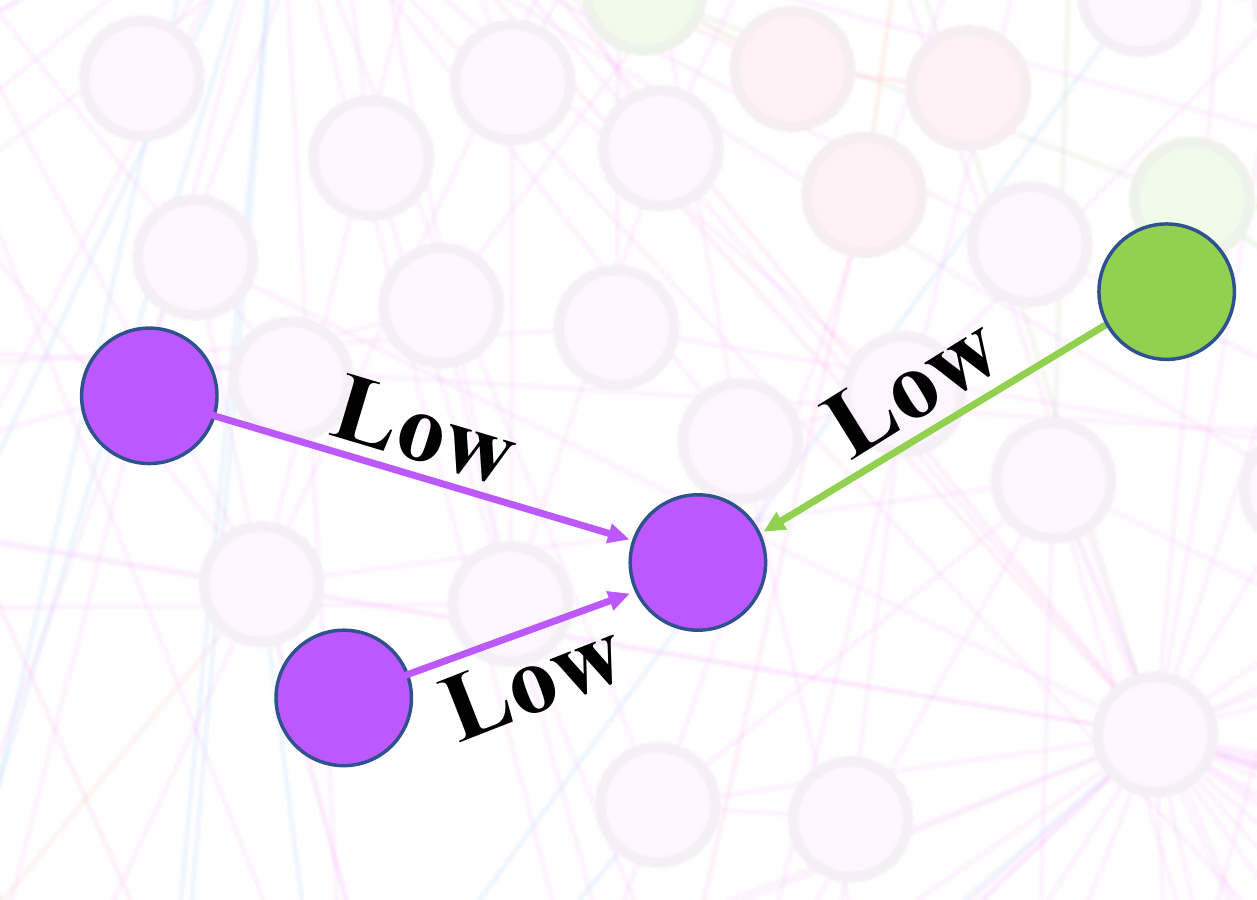}
}
\subfigure[FAGCN]{
\label{fig:FAGCN}
\includegraphics[width=0.31\textwidth]{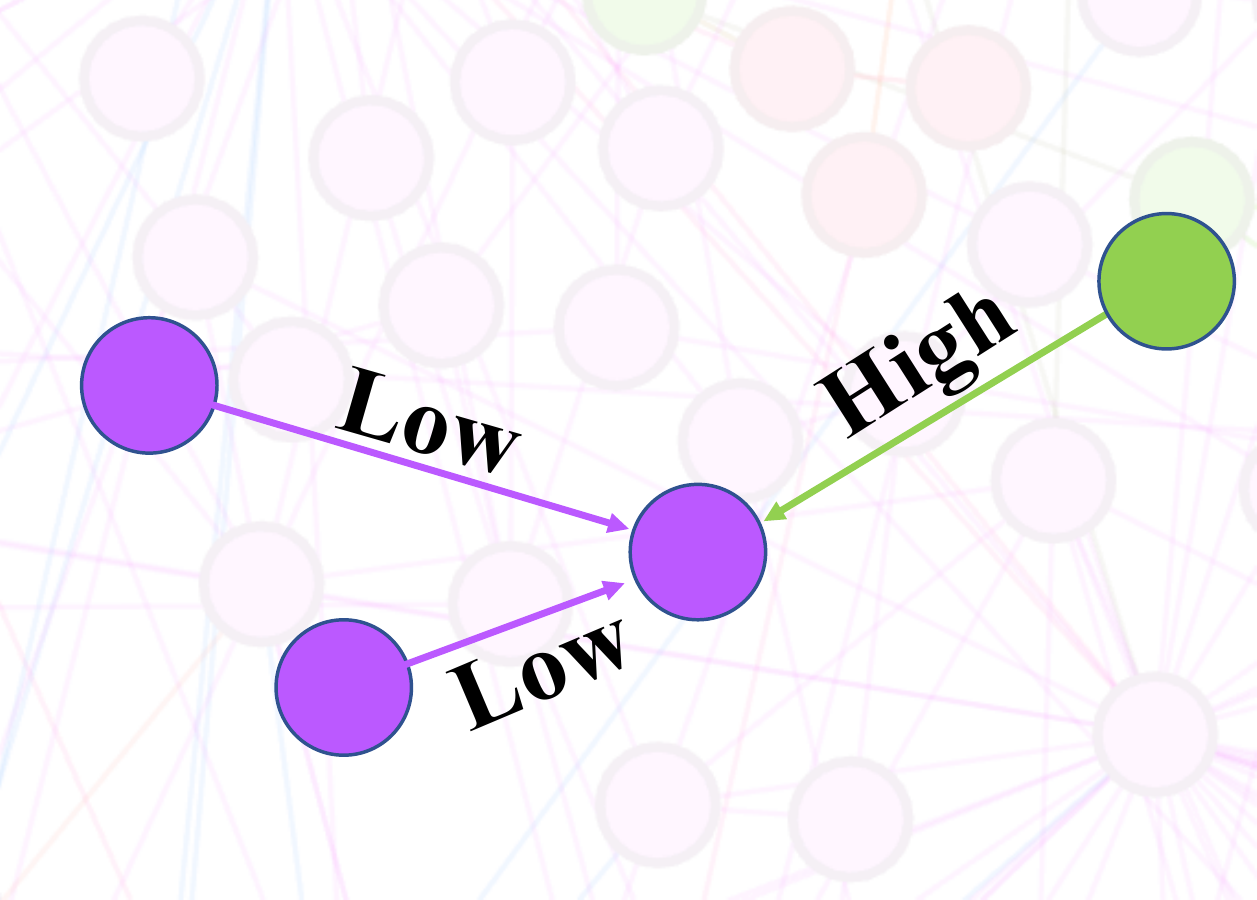}
}
\caption{(a) Classification accuracy of low-frequency signals, high-frequency signals and our model FAGCN. X-axis denotes probability of inter-connection $q$. (b) Existing GNNs aggregate the low-frequency signals of neighbors. (c) FAGCN aggregates the low-frequency signals of neighbors within the same class and high-frequency signals of neighbors from different classes, where the color indicates the node label.}
\label{intro}
\end{figure*}

To answer this question, two challenges need to be solved:
(1) Both the low-frequency and high-frequency signals are the parts of the raw features. Traditional filter is specifically designed for one certain signal, and cannot well extract different frequency signals  simultaneously.
(2) Even we can extract different information, however, the assortativity of real-world networks is usually agnostic and varies greatly, moreover, the correlation between task and different information is very complex, so it is difficult to decide what kind of signals should be used: raw features, low-frequency signals, high-frequency signals or their combination.

In this paper, we design a general frequency adaptation graph convolutional networks called FAGCN, to adaptively aggregate different signals from neighbors or itself.
We first employ the theory of graph signal processing to formally define an enhanced low-pass and high-pass filter to separate the low-frequency and high-frequency signals from the raw features.
Then we design a self-gating mechanism to adaptively integrate the low-frequency signals, high-frequency signals and raw features, without knowing the assortativity of network. 
Theoretical analysis proves that FAGCN is a generalization of most existing GNNs and it has a capability to freely shorten or enlarge the distance between node representations, which further explains why FAGCN can perform well on different types of networks.





The contribution of this paper is summarized as follows:
\begin{itemize}
	\item We study the roles of both low-frequency and high-frequency signals in GNNs and verify that high-frequency signals are useful for disassortative networks.
	\item We propose a novel graph convolutional networks FAGCN, which can adaptively change the proportion of low-frequency and high-frequency signals without knowing the types of networks.
	\item We theoretically prove that the expressive power of FAGCN is greater than other GNNs. Moreover, our proposed FAGCN is able to alleviate the over-smoothing problem. Extensive experiments on six real-world networks validate that FAGCN has advantages over state-of-the-arts.
\end{itemize}

\begin{figure*}
\centering
\subfigure[$\mathcal{F}_{L}$]{
\label{low1}
\includegraphics[width=0.23\textwidth]{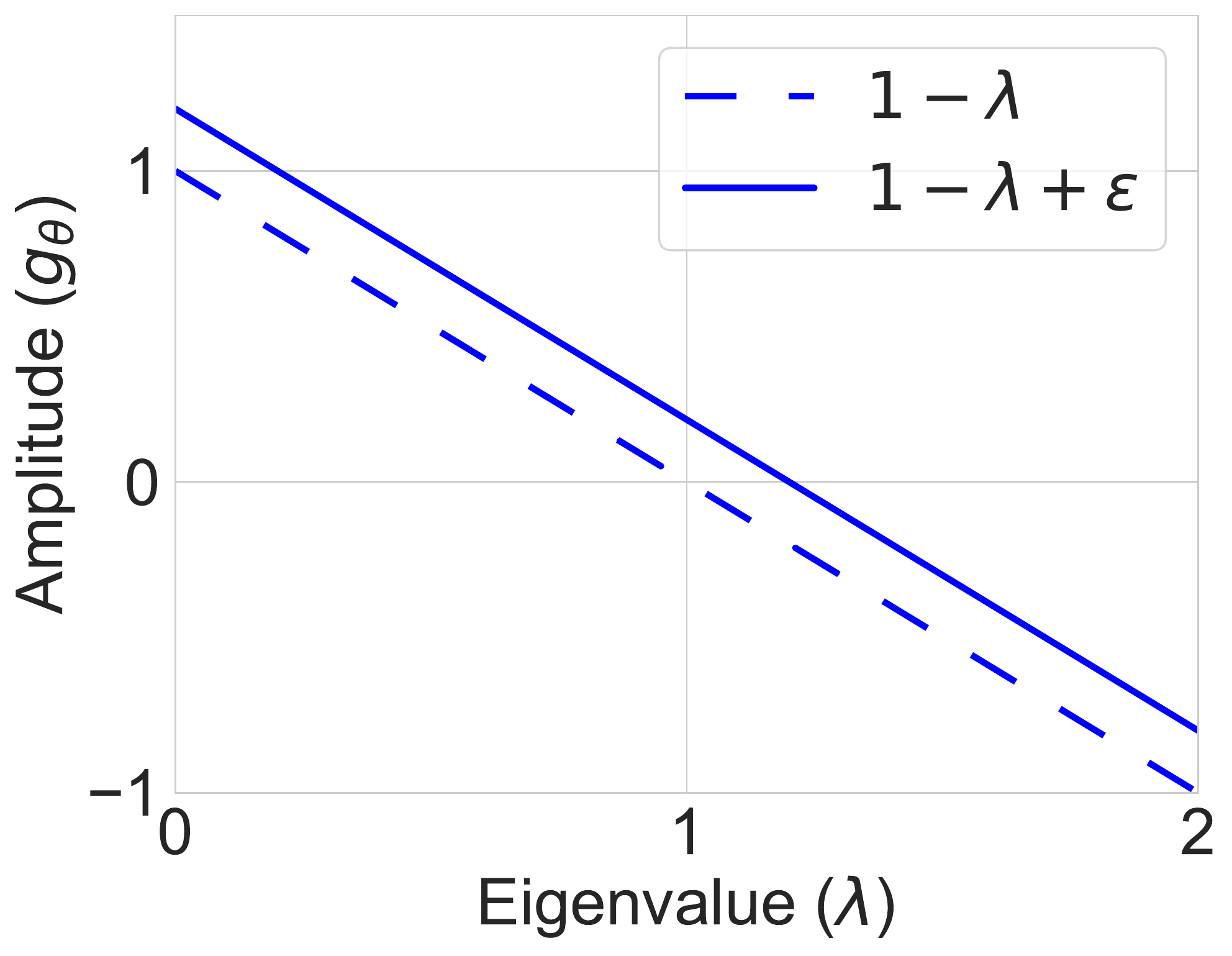}
}
\subfigure[$\mathcal{F}_{L}^{2}$]{
\label{low2}
\includegraphics[width=0.23\textwidth]{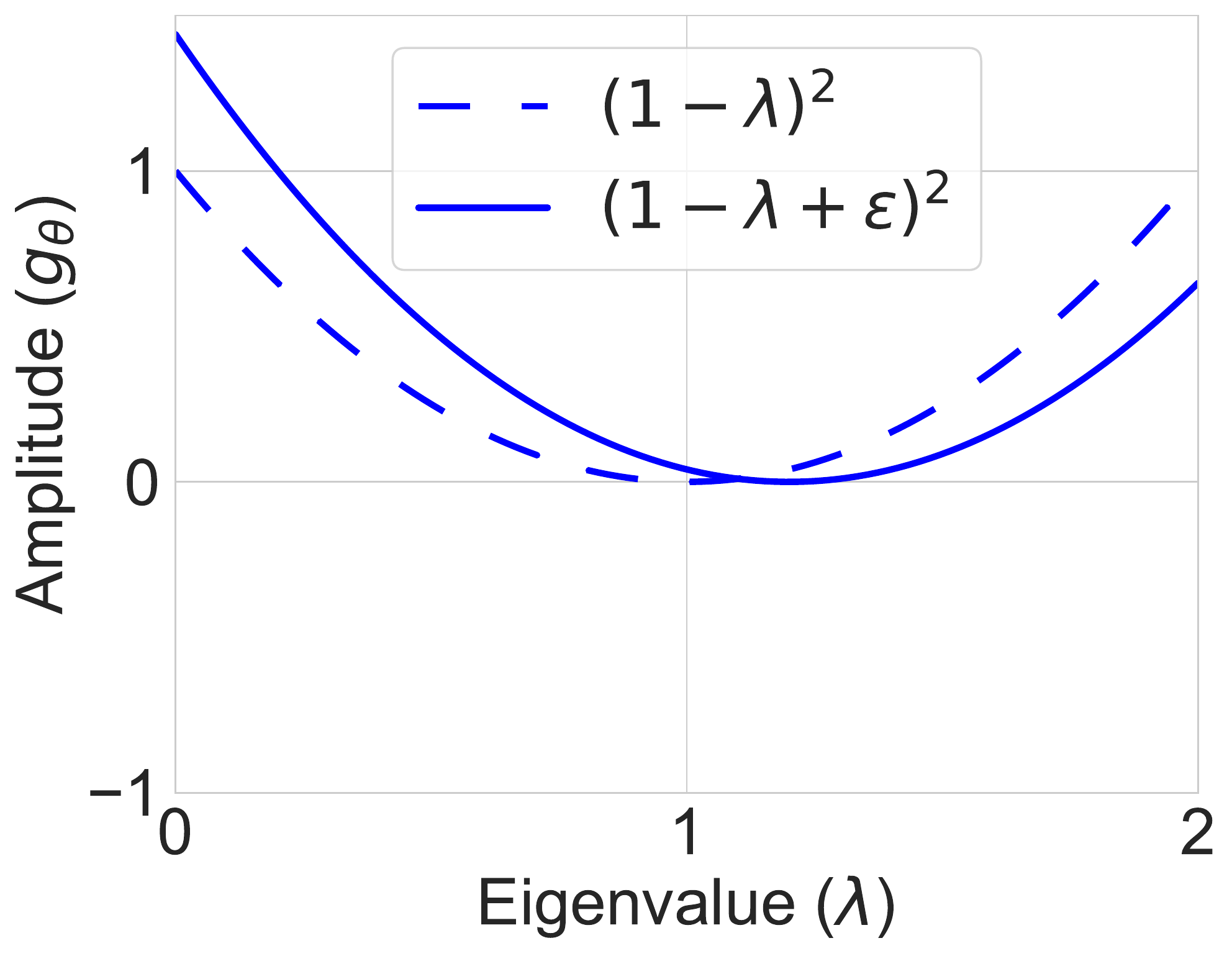}
}
\subfigure[$\mathcal{F}_{H}$]{
\label{high1}
\includegraphics[width=0.23\textwidth]{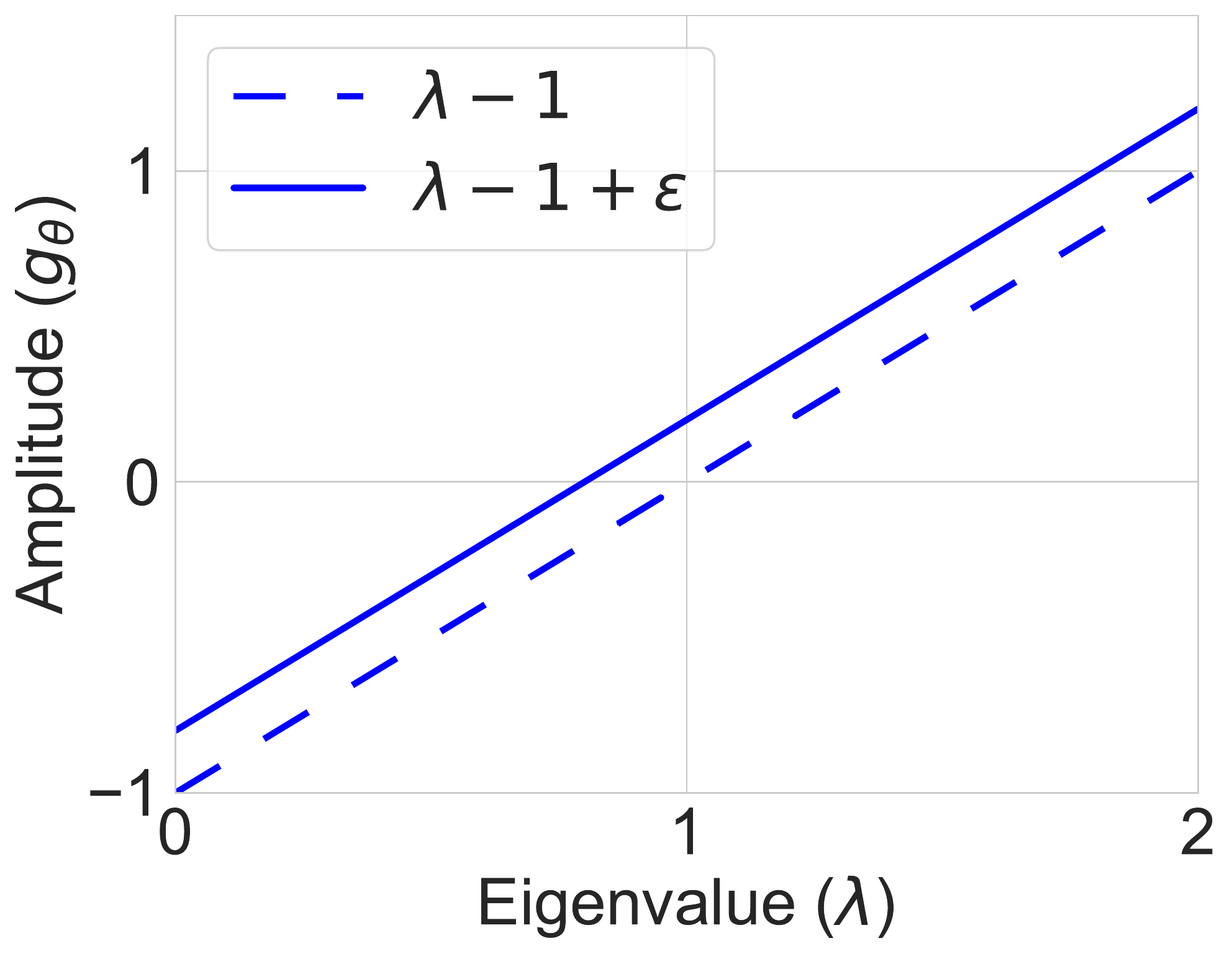}
}
\subfigure[$\mathcal{F}_{H}^{2}$]{
\label{high2}
\includegraphics[width=0.23\textwidth]{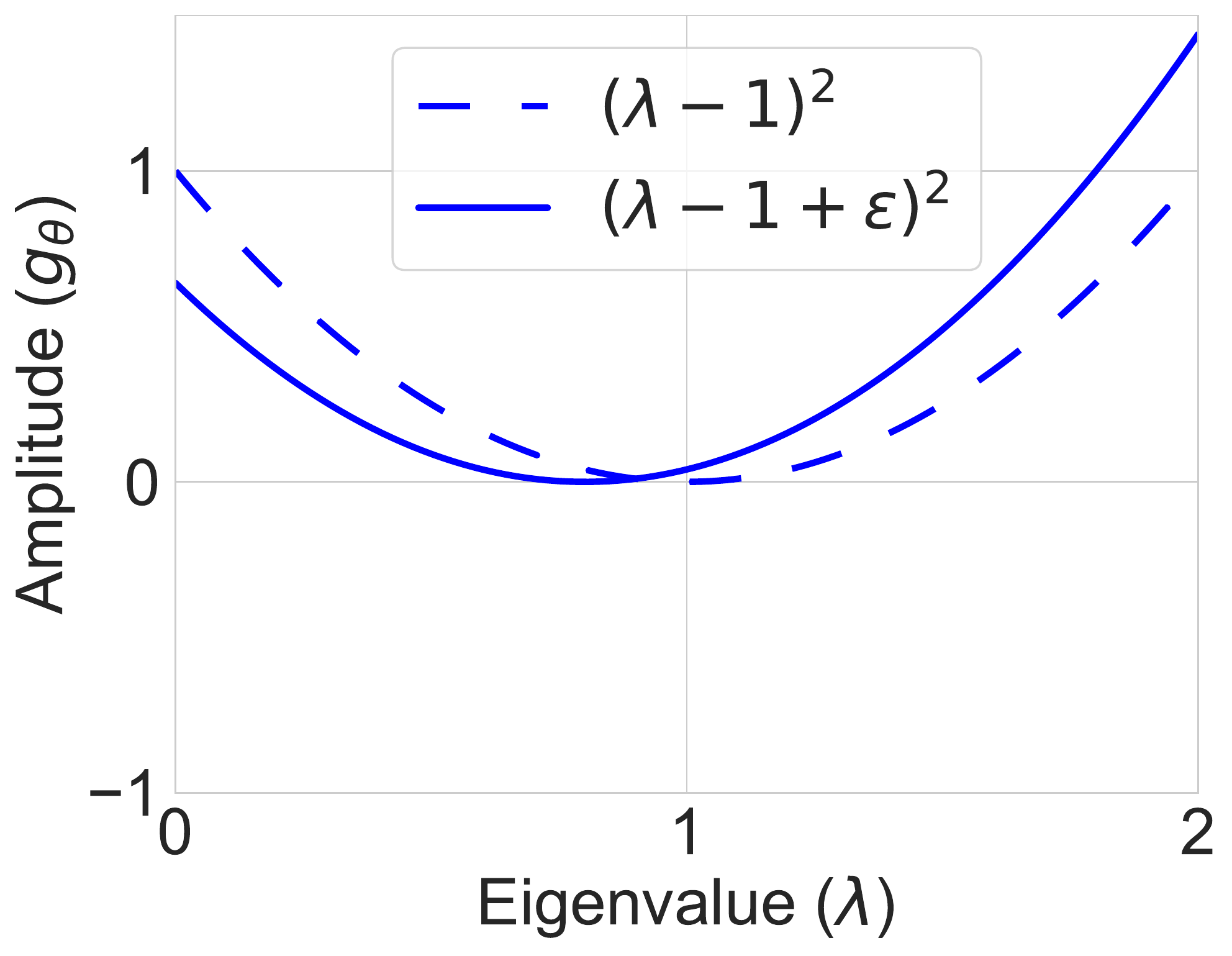}
}
\caption{The relations between eigenvalues and amplitudes in different filters.}
\label{functions}
\end{figure*}

\section{An Experimental Investigation}
\label{case}

In this section, taking the low-frequency and high-frequency signals as an example, we analyze their roles in learning node representations. Specifically, we test their performance of node classification on a series of synthetic networks.
The main idea is to gradually increase the disassortativity of the synthetic networks, and observe how the performance of these two signals changes.
We generate a network with 200 nodes and randomly divide them into 2 classes.
For each node in class one, we sample a 20-dimensional feature vector from Gaussian distribution $\mathcal{N}(0.5,1)$, while for the nodes in class two, the distribution is $\mathcal{N}(-0.5,1)$.
Besides, the connections in the same class are generated from a Bernoulli distribution with probability $p=0.05$, and the probability of connections between two classes $q$ varies from $0.01$ to $0.1$. When $q$ is small, the network exhibits assortativity; As $q$ increases, the network gradually exhibits disassortativity.
We then apply the low-pass and high-pass filters, described in Section \ref{separation}, to node classification task. Half of the nodes are used for training and the remains are used for testing.

Figure \ref{fig:vary_q} illustrates that with the increase of inter-connection $q$, the accuracy of low-frequency signals decreases, while the accuracy of high-frequency signals increases gradually.
This proves that both the low-frequency and high-frequency signals are helpful in learning node representations.
The reason why existing GNNs fail when $q$ increases is that, as shown in Figure \ref{fig:GNNs}, they only aggregate low-frequency signals from neighbors, i.e., making the node representations become similar, regardless of whether nodes belong to the same class, thereby losing the discrimination. 
When the network becomes disassortative, the effectiveness of high-frequency signals appears, but as shown in Figure \ref{fig:vary_q}, a single filter cannot achieve optimal results in all cases. Our proposed FAGCN, which combines the advantages of both low-pass and high-pass filters, can aggregate the low-frequency signals of neighbors within the same class and high-frequency signals of neighbors from different classes, as shown in Figure \ref{fig:FAGCN}, thereby showing the best performance on every synthetic network.

\section{Our Proposed Model: FAGCN}
\label{separation}

Consider an undirected graph $G=(V, E)$ with adjacency matrix $A \in \mathbb{R}^{N \times N}$, where $V$ is a set of nodes with $\left| V \right| = N$ and $E$ is a set of edges. The normalized graph Laplacian matrix is defined as $L=I_{n}-D^{-1/2}AD^{-1/2}$, where $D \in \mathbb{R}^{N \times N}$ is a diagonal degree matrix with $D_{i,i}=\sum_{j}A_{i,j}$ and $I_{n}$ denotes the identity matrix. Because $L$ is a real symmetric matrix, it has a complete set of orthonormal eigenvectors $\{u_{l}\}_{l=1}^{n} \in \mathbb{R}^{n}$, each of which has a corresponding eigenvalue $\lambda_{l} \in [0, 2]$ \cite{spectralgraph}. Through the eigenvalues and eigenvectors, we have $L=U\Lambda U^{\top}$, where $\Lambda=diag([ \lambda_{1},\lambda_{2},\cdots,\lambda_{n} ])$.

\textbf{Graph Fourier Transform.}
According to theory of graph signal processing \cite{GSP}, we can treat the eigenvectors of normalized Laplacian matrix as the bases in graph Fourier transform. Given a signal $x \in \mathbb{R}^{n}$, the graph Fourier transform is defined as $\hat{x}=U^{\top}x$, and the inverse graph Fourier transform is $x=U\hat{x}$. Thus, the convolutional $*_{G}$ between the signal $x$ and convolution kernel $f$ is:
\begin{equation}
\label{convolution}
	f *_{G} x = U \left( \left( U^{\top}f \right) \odot \left( U^{\top}x \right) \right) = Ug_{\theta}U^{\top}x,
\end{equation}
where $\odot$ denotes the element-wise product of vectors and $g_{\theta}$ is a diagonal matrix, which represents the convolutional kernel in the spectral domain, replacing $U^{\top}f$. Spectral CNN \cite{SpectralCNN} uses a non-parametric convolutional kernel $g_{\theta}=diag(\{ \theta_{i} \}_{i=1}^{n})$. ChebNet \cite{ChebNet} parameterizes convolutional kernel with a polynomial expansion $g_{\theta}=\sum_{k=0}^{K-1}\alpha_{k}\Lambda^{k}$. GCN defines the convolutional kernel as $g_{\theta}=I-\Lambda$.

\subsection{Separation}
As discussed in Section \ref{case}, both the low-frequency and high-frequency signals are helpful for learning node representations. 
To make full use of them, we design a low-pass filter $\mathcal{F}_{L}$ and a high-pass filter $\mathcal{F}_{H}$ to separate the low-frequency and high-frequency signals from the node features:
\begin{gather}
	\mathcal{F}_{L} = \varepsilon I + D^{-1/2}AD^{-1/2} = (\varepsilon+1) I - L, \notag \\
	\mathcal{F}_{H} = \varepsilon I - D^{-1/2}AD^{-1/2} = (\varepsilon-1) I + L, \label{filters}
\end{gather}
where $\varepsilon$ is a scaling hyper-parameter limited in $[0, 1]$. If we use $\mathcal{F}_{L}$ and $\mathcal{F}_{H}$ to replace the convolutional kernel $f$ in Equation \ref{convolution}. The signal $x$ is filtered by $\mathcal{F}_{L}$ and $\mathcal{F}_{H}$ as:
\begin{gather}
	\mathcal{F}_{L} *_{G} x = U[(\varepsilon+1) I - \Lambda]U^{\top}x = \mathcal{F}_{L} \cdot x, \notag \\
	\mathcal{F}_{H} *_{G} x = U[(\varepsilon-1) I + \Lambda]U^{\top}x = \mathcal{F}_{H} \cdot x.
	\label{filterings}
\end{gather}
Therefore, the convolutional kernel of $\mathcal{F}_{L}$ is $g_{\theta}=(\varepsilon+1) I - \Lambda$, rewritten as $g_{\theta}(\lambda_{i})=\varepsilon+1-\lambda_{i}$, shown in Figure \ref{low1}. When $\lambda_{i} > 1+\varepsilon$, $g_{\theta}(\lambda_{i}) < 0$, which gives a negative amplitude. 
To avoid this, we consider the second-order convolution kernel $\mathcal{F}_{L}^{2}$ with $g_{\theta}(\lambda_{i})=(\varepsilon+1-\lambda_{i})^{2}$, shown in Figure \ref{low2}. When $\lambda_{i} = 0$, $g_{\theta}(\lambda_{i}) = (\varepsilon+1)^{2} > 1$ and when $\lambda_{i} = 2$, $g_{\theta}(\lambda_{i}) = (\varepsilon-1)^{2} < 1$, which amplifies the low-frequency signals and restrains the high-frequency signals.
\begin{remark}
	(Enhanced filters) As in Figure \ref{functions}, compared with traditional low-pass filters, e.g., GCN and SGC \cite{SGC}, $\mathcal{F}_{L}$ is an enhanced low-pass filter. Convolutional kernel of second-order GCN is $g_{\theta}(\lambda_{i})=(1-\lambda_{i})^{2}$. When $\lambda_{i}=0$, the amplitude of GCN is $g_{\theta}(\lambda_{i})=1<(1+\varepsilon)^{2}$. Hence, the value of $\mathcal{F}_{L}$ is greater than GCN in low-pass filtering. Similarly, $\mathcal{F}_{H}$ is an enhanced high-pass filter, which provides a greater value for the high-frequency signals.
\end{remark}

Separating the low-frequency and high-frequency signals from the node features provides a feasible way to deal with different networks, e.g., low-frequency signals for assortative networks and high-frequency signals for disassortative networks.
However, this way has two disadvantages: One is that selecting signals requires a priori knowledge, i.e., we actually do not know whether a network is assortative or disassortative beforehand. The other is that, as in Equation \ref{filterings}, it requires matrix multiplication, which is undesirable for large graphs \cite{GraphSAGE}.
Therefore, an efficient method that can adaptively aggregate low-frequency and high-frequency signals is desired.


\begin{remark}
\label{rema}
	(Concrete meaning of signals)
	In Equation \ref{filters}, we have $\mathcal{F}_{L} = \varepsilon I + D^{-1/2}AD^{-1/2}$ and $\mathcal{F}_{H} = \varepsilon I - D^{-1/2}AD^{-1/2}$. 
	Therefore, the concrete meaning of low-frequency signal $\mathcal{F}_{L} \cdot x$ is the sum of node features and neighborhood features in spatial domain, while high-frequency signal $\mathcal{F}_{H} \cdot x$ represents the difference between node features and neighborhood features in spatial domain.
\end{remark}

\begin{figure}
\centering
\includegraphics[width=\linewidth]{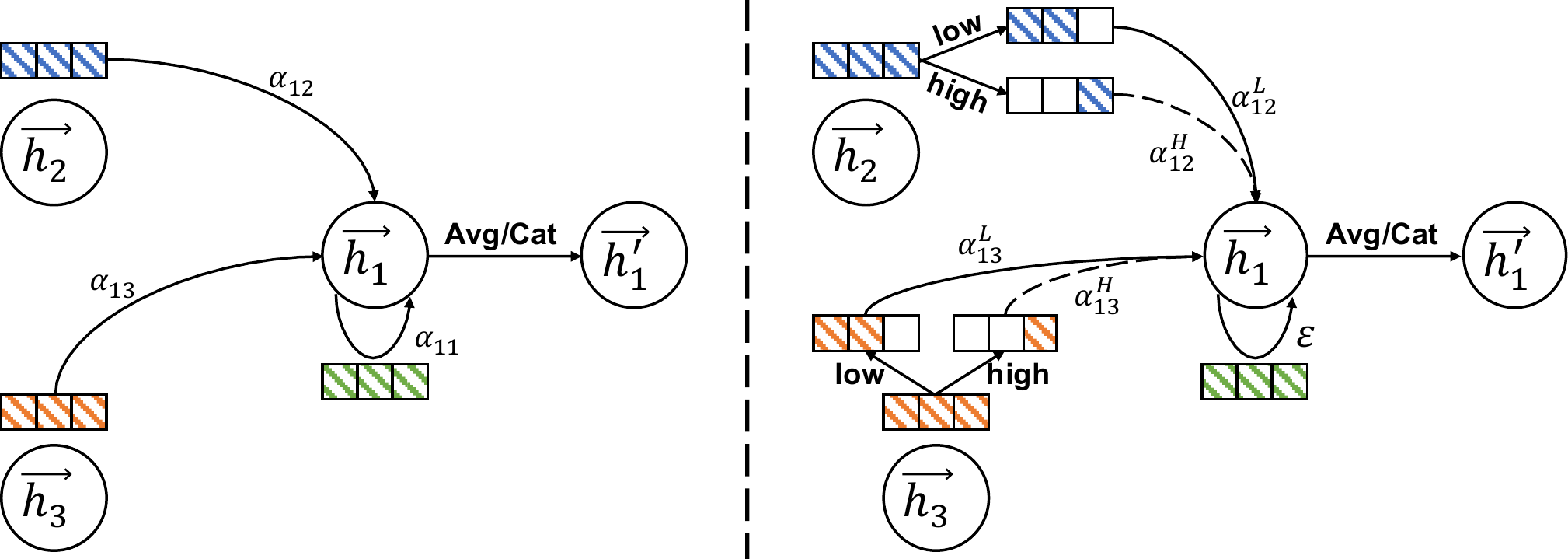}
\caption{\textbf{Left:} The aggregation process of existing GNNs, and $\alpha_{ij}$ indicates the importance of node $j$ to node $i$. \textbf{Right:} The aggregation process of FAGCN, and $\alpha_{ij}^{L}$, $\alpha_{ij}^{H}$ denote the proportions of low-frequency and high-frequency signals of node $j$ to node $i$, respectively.}
\label{model}
\end{figure}

\subsection{Aggregation}
Before introducing the details, we first compare the aggregation process of existing GNNs and FAGCN in Figure \ref{model}. The left shows that existing GNNs consider learning the importance ($\alpha_{ij}$) of each node in aggregation. The right is FAGCN that uses two coefficients ($\alpha^{L}_{ij}$ and $\alpha^{H}_{ij}$) to aggregate low-frequency and high-frequency signals from the neighbors, respectively.

The input of our model are the node features, $\mathbf{H}=\{ \mathbf{h}_{1},\mathbf{h}_{2},\cdots,\mathbf{h}_{N} \} \in \mathbb{R}^{N \times F}$, where $F$ is the dimension of the node features.
For the purpose of frequency adaptation, a basic idea is to use the attention mechanism to learn the proportion of low-frequency and high-frequency signals:
\begin{equation}
	\tilde{\mathbf{h}}_{i} = \alpha_{ij}^{L} (\mathcal{F}_{L} \cdot \mathbf{H})_{i} + \alpha_{ij}^{H} (\mathcal{F}_{H} \cdot \mathbf{H})_{i} = \varepsilon \mathbf{h}_{i} + \sum_{j \in \mathcal{N}_{i}} \frac{\alpha_{ij}^{L} - \alpha_{ij}^{H}}{\sqrt{d_{i}d_{j}}} \mathbf{h}_{j},
\end{equation}
where $\tilde{\mathbf{h}}_{i}$ is the aggregated representation of node $i$. $\mathcal{N}_{i}$ and $d_{i}$ denote the neighbor set and degree of node $i$, respectively. $\alpha_{ij}^{L}$ and $\alpha_{ij}^{H}$ represent the proportions of node $j$'s low-frequency and high-frequency signals to node $i$. 
We set $\alpha_{ij}^{L} + \alpha_{ij}^{H}=1$ and $\alpha_{ij}^{G} = \alpha_{ij}^{L} - \alpha_{ij}^{H}$.
In the following, we show that $\alpha_{ij}^{G}$ can be interpreted from two perspectives.

\begin{remark}
	(Two perspectives of $\alpha_{ij}^{G}$) One is that $\alpha_{ij}^{G}$ indirectly represents the proportion of low-frequency and high-frequency signals.  $\alpha_{ij}^{G} > 0$, i.e., $\alpha_{ij}^{L} > \alpha_{ij}^{H}$, means that low-frequency signals dominate the representations and vice versa. 
	Based on $\alpha_{ij}^{G}$, we can calculate the value of $\alpha_{ij}^{L}$ and $\alpha_{ij}^{H}$, so as to achieve the proportions of signals.
	Another is that $\alpha_{ij}^{G}$ denotes the coefficients of neighbors in aggregation. $\alpha_{ij}^{G} > 0$ represents the sum of node features and neighborhood features, i.e., $\mathbf{h}_{i} + \mathbf{h}_{j}$, while $\alpha_{ij}^{G} < 0$ represents the difference between them, i.e., $\mathbf{h}_{i} - \mathbf{h}_{j}$, as explained in Remark \ref{rema}. Besides, when $\alpha_{ij}^{G} \approx 0$, the contributions of neighbors will be limited, so the raw features will dominate the node representations. 
\end{remark}

In order to learn the coefficients $\alpha_{ij}^{G}$ effectively, we need to consider the features of both the node itself and its neighbors.
Therefore, we propose a shared \emph{self-gating} mechanism $\mathbb{R}^{F} \times \mathbb{R}^{F} \to \mathbb{R}$ to learn the coefficients:
\begin{equation}
\label{tanh}
	\alpha_{ij}^{G}=\tanh \left( \mathbf{g}^{\top} \left[ \mathbf{h}_{i} \parallel \mathbf{h}_{j} \right] \right),
\end{equation}
where $\parallel$ denotes the concatenation operation, $\mathbf{g} \in \mathbb{R}^{2F}$ can be seen as a shared convolutional kernel \cite{GAT} and $\tanh(\cdot)$ is the hyperbolic tangent function, which can naturally limits the value of $\alpha_{ij}^{G}$ in $[-1, 1]$. Besides, to make use of the structural information, we only calculate the coefficients among the node and its first-order neighbors $\mathcal{N}_{i}$.

After calculating $\alpha_{ij}^{G}$, we can aggregate the representations of neighbors:
\begin{equation}
\label{self-gating}
	\mathbf{h}_{i}^{'}= \varepsilon \mathbf{h}_{i} + \sum_{j \in \mathcal{N}_{i}} \frac{\alpha_{ij}^{G}}{\sqrt{d_{i}d_{j}}} \mathbf{h}_{j},
\end{equation}
where $\mathbf{h}_{i}^{'}$ denotes the aggregated representation of node $i$. Note that when aggregating information from neighbors, the degrees are used to normalize the coefficients, thus preventing the aggregated representations from being too large. 
\subsection{The Whole Architecture of FAGCN}
In the previous section, we introduce the message passing process of FAGCN. Here, we formally define the whole architecture of FAGCN. Some recent studies \cite{SGC, cui2020adaptive} emphasize that the entanglement of filters and weight matrices may be harmful to the performance and robustness of the model. Motivated by this, we first use a multilayer perceptron (MLP) to apply the non-linear transform to the raw features. Then we propagate the representations through Eq. \ref{self-gating}. The mathematical expression of FAGCN is defined as:
\begin{align}
    \mathbf{h}^{(0)}_{i} & = \phi(\mathbf{W}_{1}\mathbf{h}_{i}) & \in \mathbb{R}^{F' \times 1} \notag \\
    \mathbf{h}^{(l)}_{i} & = \varepsilon \mathbf{h}^{(0)}_{i} + \sum_{j \in \mathcal{N}_{i}} \frac{\alpha_{ij}^{G}}{\sqrt{d_{i}d_{j}}} \mathbf{h}^{(l-1)}_{j} & \in \mathbb{R}^{F' \times 1} \notag \\
    \mathbf{h}_{out} & = \mathbf{W}_{2}\mathbf{h}^{(L)}_{i} & \in \mathbb{R}^{K \times 1},
\end{align}
where $\mathbf{W}_{1} \in \mathbb{R}^{F \times F'}, \mathbf{W}_{2} \in \mathbb{R}^{F' \times K}$ are the weight matrices, $\phi$ is the activation function, $F^{'}$ denotes the dimension of hidden layers, $l$ indicates the layers, ranging from 1 to $L$, and $K$ represents the number of classes.
The complexity of a single layer FAGCN is $\mathcal{O}((N + |E|) \times F')$, which is approximately linear with the number of edges and nodes.

\section{Theoretical Analysis}
\label{theory}

\subsection{Connection to Existing GNNs}

FAGCN is a generalization of most existing GNNs. Specifically, when we set the coefficients $\alpha_{ij}^{G}$ to 1, FAGCN acts like GCN and when we use softmax function to normalize $\alpha_{ij}^{G}$, FAGCN becomes GAT.
Therefore, as indicated in Remark \ref{rema}, because the coefficients in GCN and GAT are both greater than zero, they prefer to aggregate the low-frequency signals. However, FAGCN can learn a coefficient that can be positive or negative, to adaptively aggregate low-frequency and high-frequency signals.

\subsection{Expressive Power of FAGCN}
\label{sec:expressive}

We analyze the expressive power of FAGCN from the perspective of the distance between node representations.
Assume that $(u, v)$ is a pair of connected nodes, and $\mathbf{h}_{u}, \mathbf{h}_{v}$ are the node features. Let $\mathcal{D}, \mathcal{D}_{L}, \mathcal{D}_{H}$ be the distance of node features, low-frequency signals of node features and high-frequency signals of node features, respectively.
\begin{align}
	\mathcal{D} = \| \mathbf{h}_{u}-\mathbf{h}_{v} \|_{2}. \notag \\
	\mathcal{D}_{L} = \| (\varepsilon\mathbf{h}_{u}+\mathbf{h}_{v}) - (\varepsilon\mathbf{h}_{v}+\mathbf{h}_{u}) \|_{2} &= | 1-\varepsilon |\mathcal{D}. \notag \\
	\mathcal{D}_{H} = \| (\varepsilon\mathbf{h}_{u}-\mathbf{h}_{v}) - (\varepsilon\mathbf{h}_{v}-\mathbf{h}_{u}) \|_{2} &= | 1+\varepsilon |\mathcal{D}. \notag	
\end{align}


\begin{proposition}
\label{prop1}
	Low-pass filtering makes the representations become similar, while high-pass filtering makes the representations become discriminative.
\end{proposition}

\begin{proof}
	It is easy to see that $\mathcal{D}_{H} > \mathcal{D} > \mathcal{D}_{L}$. 
	This indicates that compared with the original distance $\mathcal{D}$, the distance $\mathcal{D}_{L}$ induced by low-frequency signals is smaller, implying that low-frequency signals can make the representations of connected nodes become similar. 
	While the distance $\mathcal{D}_{H}$ induced by high-frequency signals is larger, implying that high-frequency signals can make the representations of connected nodes become discriminative.
\end{proof}

We have analyzed the roles of low-frequency and high-frequency signals in representation learning. Obviously, FAGCN can choose to shorten or enlarge the distance between node representations flexibly, while most existing GNNs cannot.

\begin{proposition}
	Most existing GNNs, e.g., GCN, only have the capability to make representations of nodes become similar.
\end{proposition}

\begin{proof}
	The filter used in GCN is: $(D+I)^{-1/2}(A+I)(D+I)^{-1/2}$. 
	Hence, the distance of representations learned by GCN is: $\mathcal{D}_{G} \approx \| (\frac{1}{d_{u}} \mathbf{h}_{u} + \mathbf{h}_{v}) - (\frac{1}{d_{v}} \mathbf{h}_{v} + \mathbf{h}_{u}) \|_{2} \approx |1-\frac{1}{d}|\mathcal{D} < \mathcal{D}$ (s.t. $d_{u} \approx d_{v} \approx d$).
\end{proof}

\section{Experiments}

\subsection{Datasets}
\label{datasets}

\begin{table}
  \centering
  \caption{The statistics of datasets}
	\resizebox{\linewidth}{!}{
    \begin{tabular}{lcrrrrr}
    \toprule
    \textbf{Dataset} & \textbf{Assortivity} & \textbf{Nodes} & \textbf{Edges} & \textbf{Classes} & \textbf{Features} \\
    \midrule
    Cora  & 0.771 & 2,708 & 5,429 & 7 & 1,433 \\
    Citeseer & 0.671 & 3,327 & 4,732 & 6 & 3,703 \\
    Pubmed & 0.686 & 19,717 & 44,338 & 3 & 500	\\
	\midrule
    Chameleon	& 0.180 & 2,277 & 36,101 & 3 & 2,325 \\
    Squirrel	& 0.018 & 5,201 & 217,073 & 3 & 2,089 \\
    Actor		& 0.003 & 7,600 & 33,544 & 5 & 931	\\
    \bottomrule
    \end{tabular}}
\label{statistic}
\end{table}
    
\begin{figure*}
\centering
\subfigure[Chameleon]{
\label{Chameleon}
\includegraphics[width=0.31\textwidth]{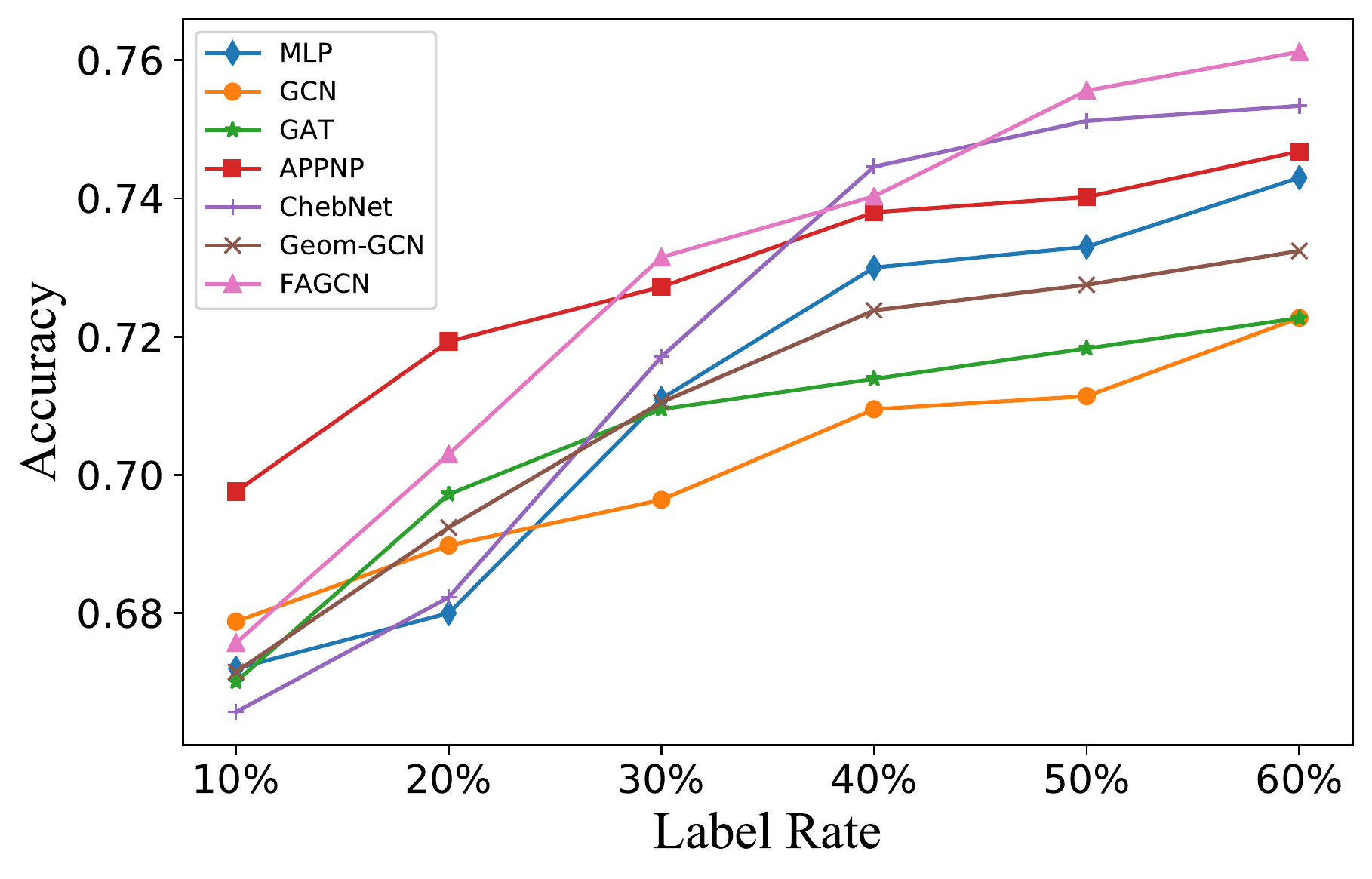}
}
\subfigure[Squirrel]{
\label{Squirrel}
\includegraphics[width=0.31\textwidth]{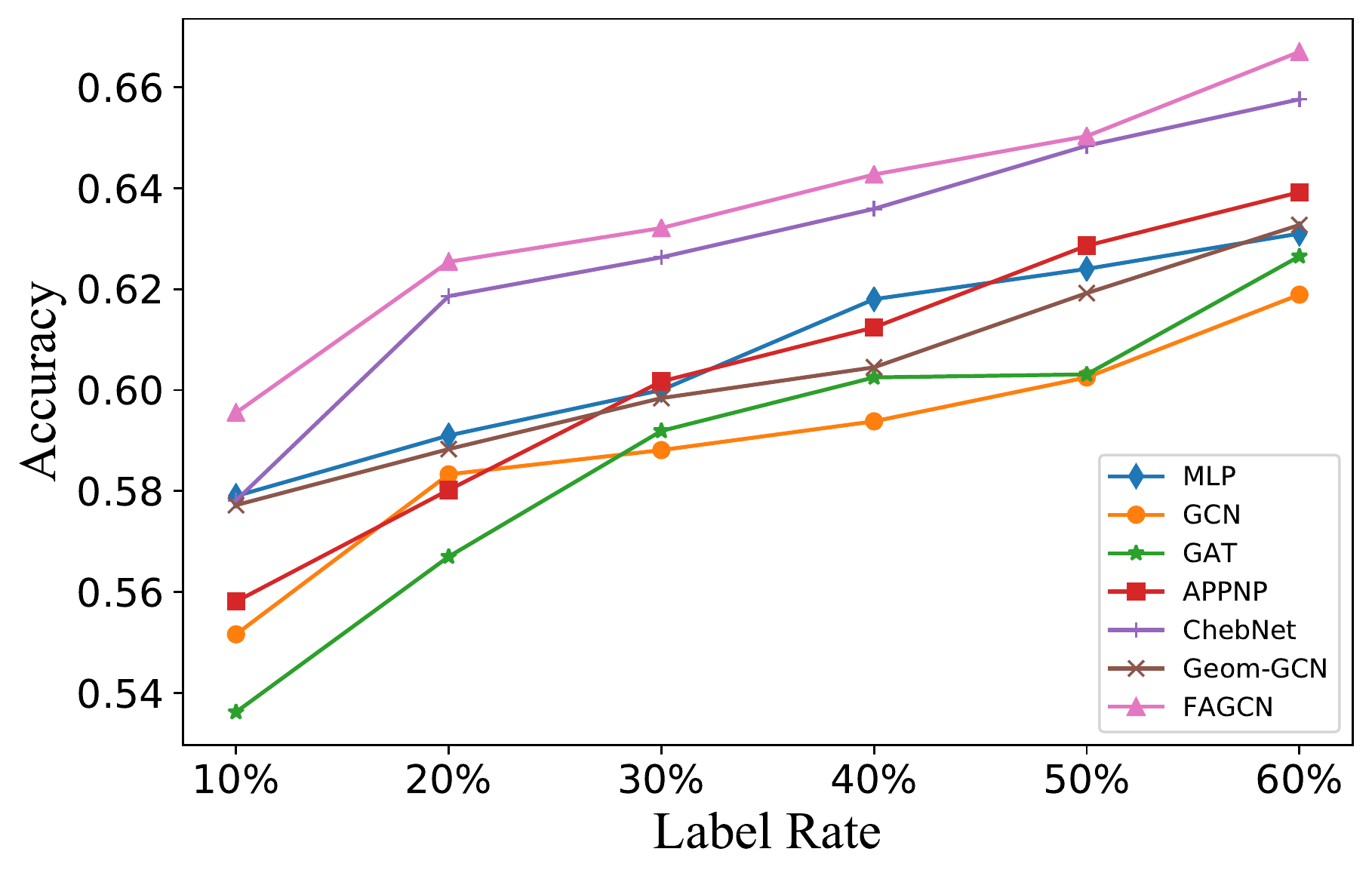}
}
\subfigure[Actor]{
\label{Actor}
\includegraphics[width=0.31\textwidth]{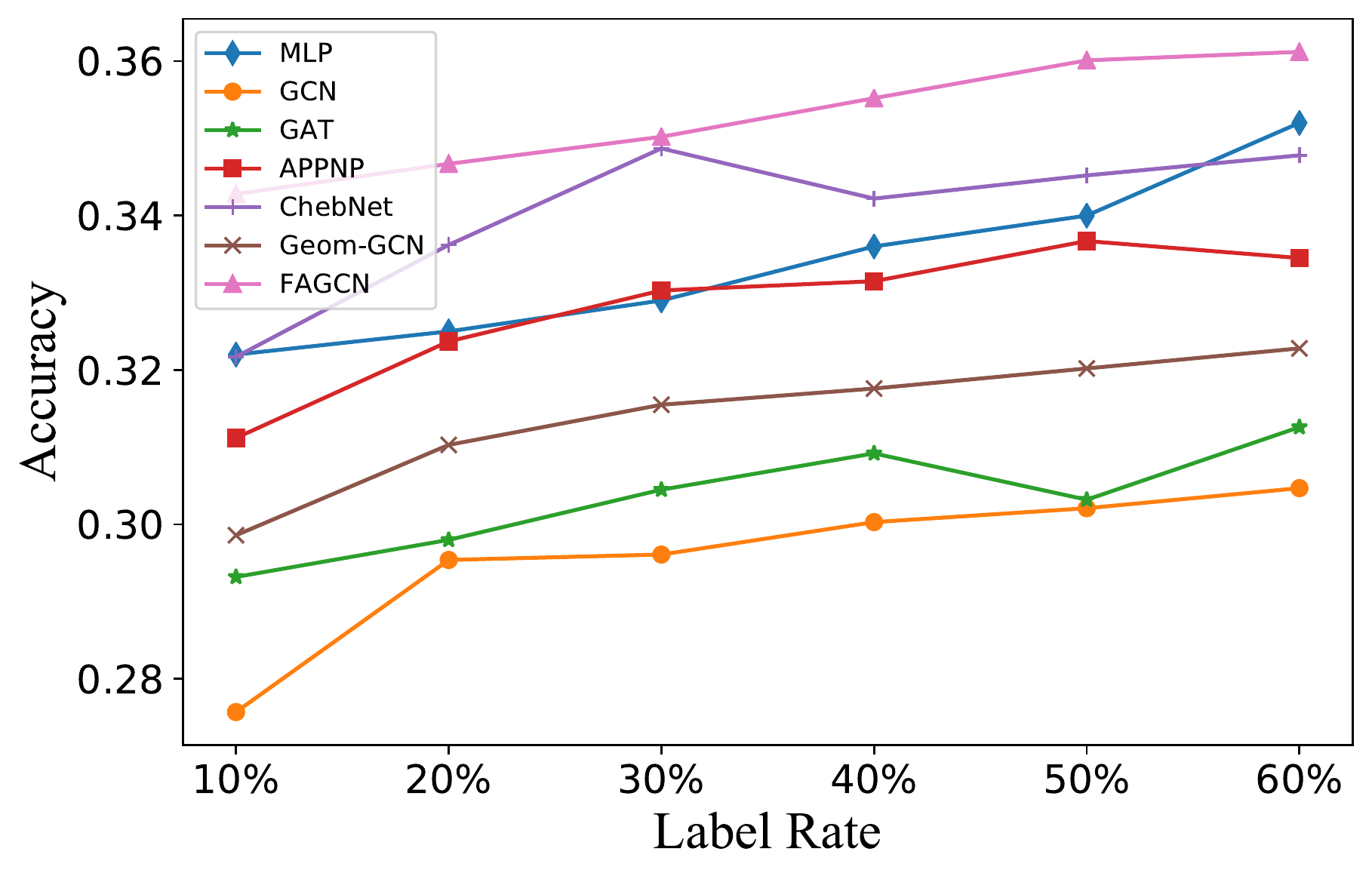}
}
\caption{Classification accuracy of different methods under different label rates on disassortative networks.}
\label{disassortative}
\end{figure*}    

\textbf{Assortative datasets.} We choose the commonly used \emph{citation networks}, e.g., Cora, Citeseer and Pubmed for assortative datasets. Edges in these networks represent the citation relationship between two papers (undirected), node features are the bag-of-words vector of the papers and labels are the fields of papers. In each network, we use 20 labeled nodes per class for training, 500 nodes for validation and 1000 nodes for testing. Details can be found in \cite{GCN}.

\noindent
\textbf{Disassortative datasets.} We consider the \emph{Wikipedia networks}\footnote{http://snap.stanford.edu/data/wikipedia-article-networks.html} and \emph{Actor co-occurrence network} \cite{actornet} for disassortative datasets. 
Chameleon and Squirrel are two Wikipedia networks. Edges represent the hyperlinks between two pages, node features are some informative nouns in the pages and labels correspond to the traffic of the pages.
In Actor co-occurrence network, each node represents an actor, and the edges denote the collaborations of them. Node features are the keywords in Wikipedia and labels are the types of actors. 
Since there is no standard division for these networks. To verify the effectiveness and robustness, we use 20\% for validation, 20\% for testing and change the training ratio from 10\% to 60\%. 

More detailed characteristics of the datasets can be found in Table \ref{statistic}. Note that a higher value of the second column represents a more obvious assortativity \cite{mixing}.

\subsection{Experimental Setup}
\label{setup}

We compare FAGCN with two types of representative GNNs: Spectral-based methods, i.e., SGC \cite{SGC},  GCN \cite{GCN}, ChebNet \cite{ChebNet} and GWNN \cite{GWNN}; Spatial-based methods, i.e., GIN \cite{GIN}, GAT \cite{GAT}, MoNet \cite{MoNet}, GraphSAGE \cite{GraphSAGE} and APPNP \cite{PPNP}. For disassortative networks, we add Geom-GCN \cite{GeomGCN} and MLP as new benchmarks.
All methods were implemented in Pytorch with Adam optimizer \cite{Adam}.
We run 10 times and report the mean values with standard deviation. The hidden unit is fixed at 16 in assortative networks and 32 in disassortative networks.
The hyper-parameter search space is: learning rate in \{0.01, 0.005\}, dropout in \{0.4, 0.5, 0.6\}, weight decay in \{1$E$-3, 5$E$-4, 5$E$-5\}, number of layers in \{1, 2, $\cdots$, 8\}, $\varepsilon$ in \{0.1, $\cdots$, 1.0\}.

In assortative datasets, we use hyper-parameters in previous literature for baselines. For FAGCN, the hyper-parameter setting is: learning rate = 0.01, dropout = 0.6, weight decay = 1$E$-3, layers = 4. $\varepsilon$ = 0.2, 0.3, 0.3 for Cora, Citeseer and Pubmed. The patience of early stop is set to 100.
In disassortative datasets, the hyper-parameter for FAGCN is: learning rate = 0.01, dropout = 0.5, weight decay = 5$E$-5, layers = 2. $\varepsilon$ = 0.4, 0.3, 0.5 for Chameleon, Squirrel and Actor, respectively.
Besides, we run 500 epochs and choose the model with highest validation accuracy for testing.

\subsection{Classification on Different Types of Networks}
\label{sec:results}

\newsavebox{\tablebox}
\begin{table}
  \centering
  \caption{Summary of node classification results (in percent).}
	\begin{lrbox}{\tablebox}
    \begin{tabular}{lcccccc}
    \toprule
    \textbf{Method} & \textbf{Cora} & \textbf{Citeseer} & \textbf{Pubmed} \\
    \midrule
	SGC & 81.0\% & 71.9\% & 78.9\% \\
    GCN & 81.5\% & 70.3\% & 79.0\% \\
    GWNN & 82.8\% & 71.7\% & 79.1\%\\
    ChebNet & 81.2\% & 69.8\% & 74.4\%\\
    GraphHeat & 83.7\% & 72.5\% & \textbf{80.5\%}\\
	\midrule
	GIN & 77.6\% & 66.1\% & 77.0\%\\ 
	GAT & 83.0\% & 72.5\% & 79.0\%\\
	MoNet & 81.7\% & -     & 78.8\%\\
	APPNP & 83.7\% & 72.1\% & 79.2\%\\
    GraphSAGE & 82.3\% & 71.2\% & 78.5\%\\
    \midrule
	\textbf{FAGCN} & \textbf{84.1$\pm$0.5\%} & \textbf{72.7$\pm$0.8\%} & 79.4$\pm$0.3\% \\
    \bottomrule
    \end{tabular}
    \end{lrbox}
    \scalebox{1}{\usebox{\tablebox}}
  \label{assortative}
\end{table}

\begin{figure}[t]
    \centering
    \subfigure[Cora]{
        \includegraphics[width=0.45\linewidth]{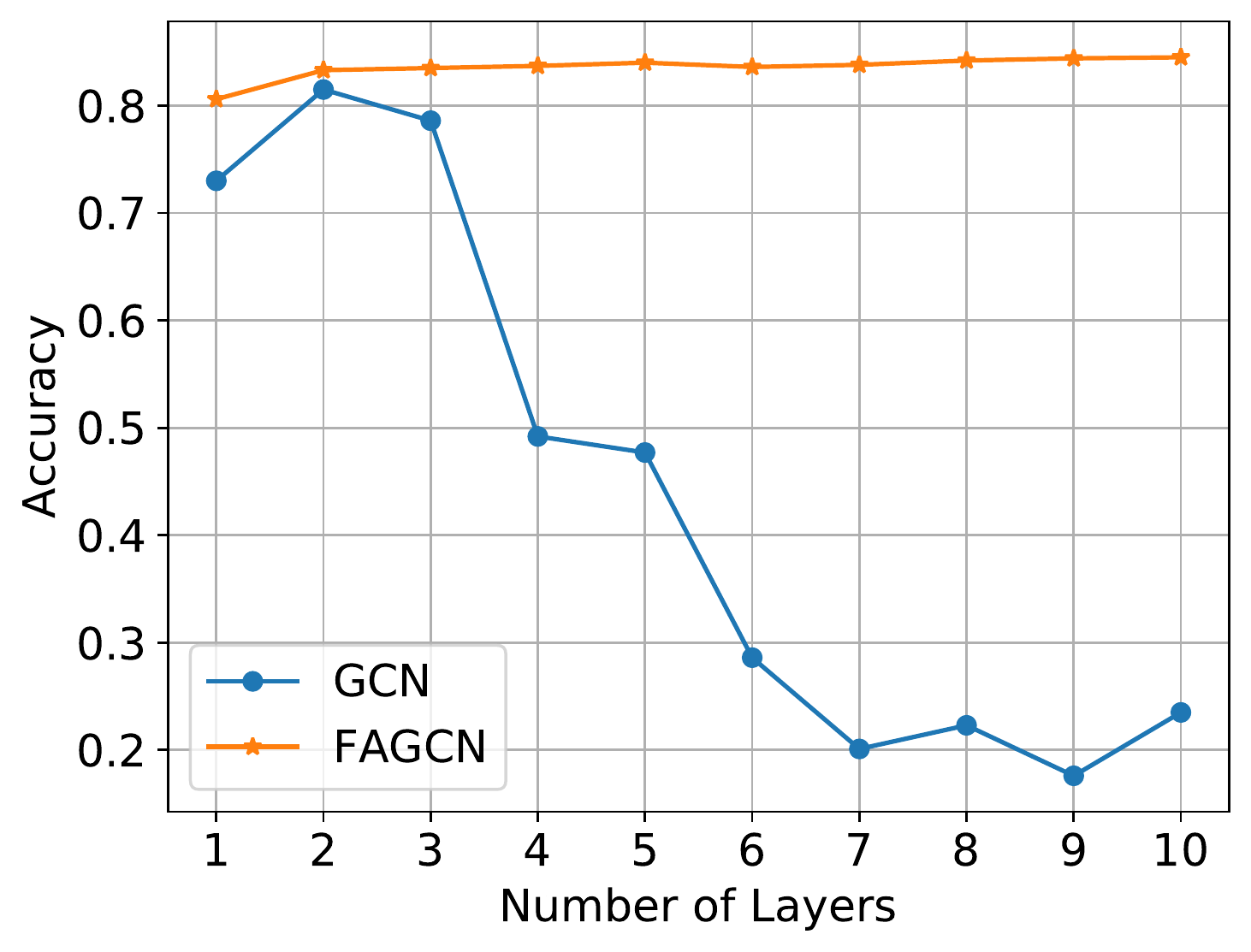}
        \label{cora_o}
    }
	\subfigure[Citeseer]{
        \includegraphics[width=0.45\linewidth]{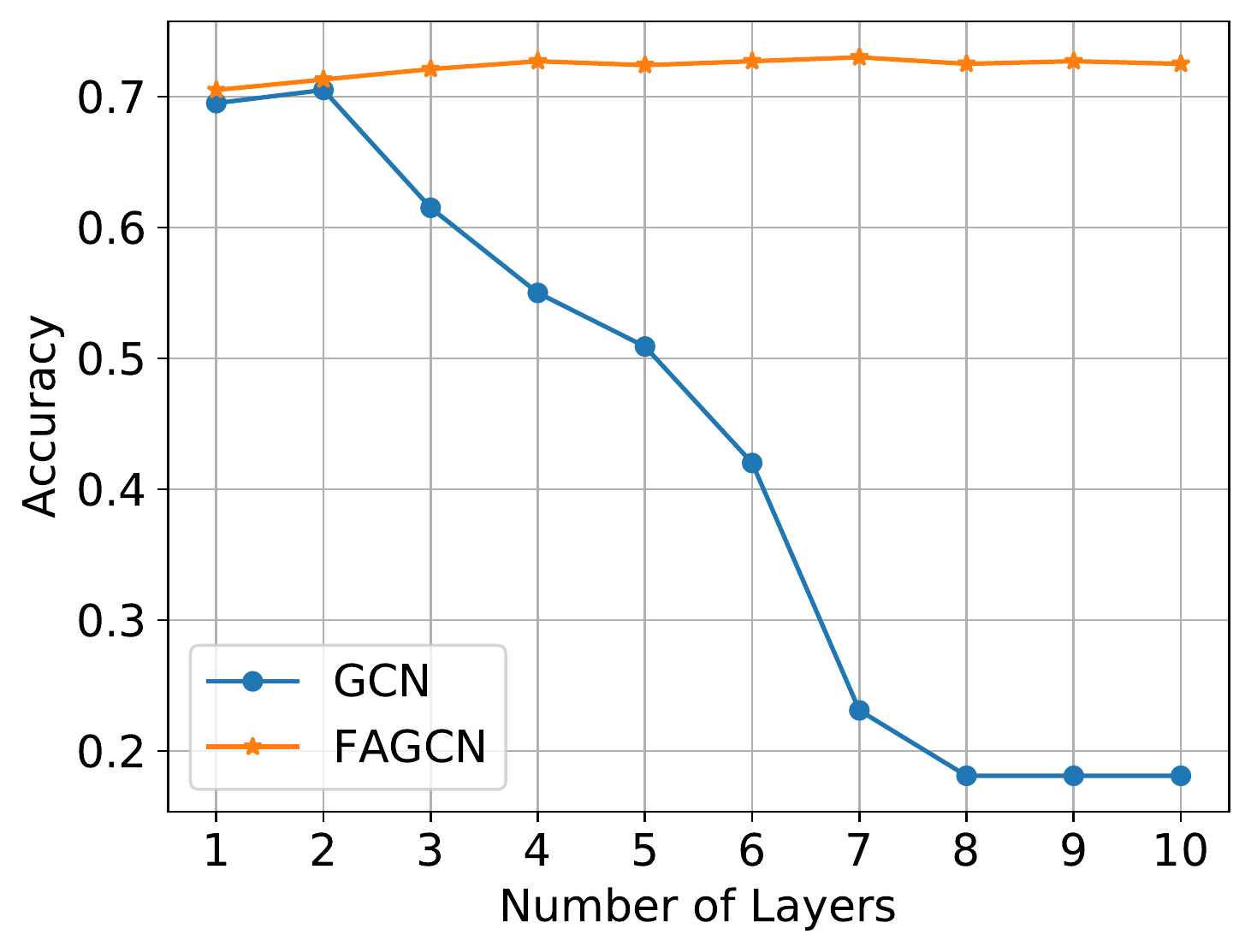}
        \label{citeseer_o}
    }
    \quad    
    \subfigure[Pubmed]{
        \includegraphics[width=0.45\linewidth]{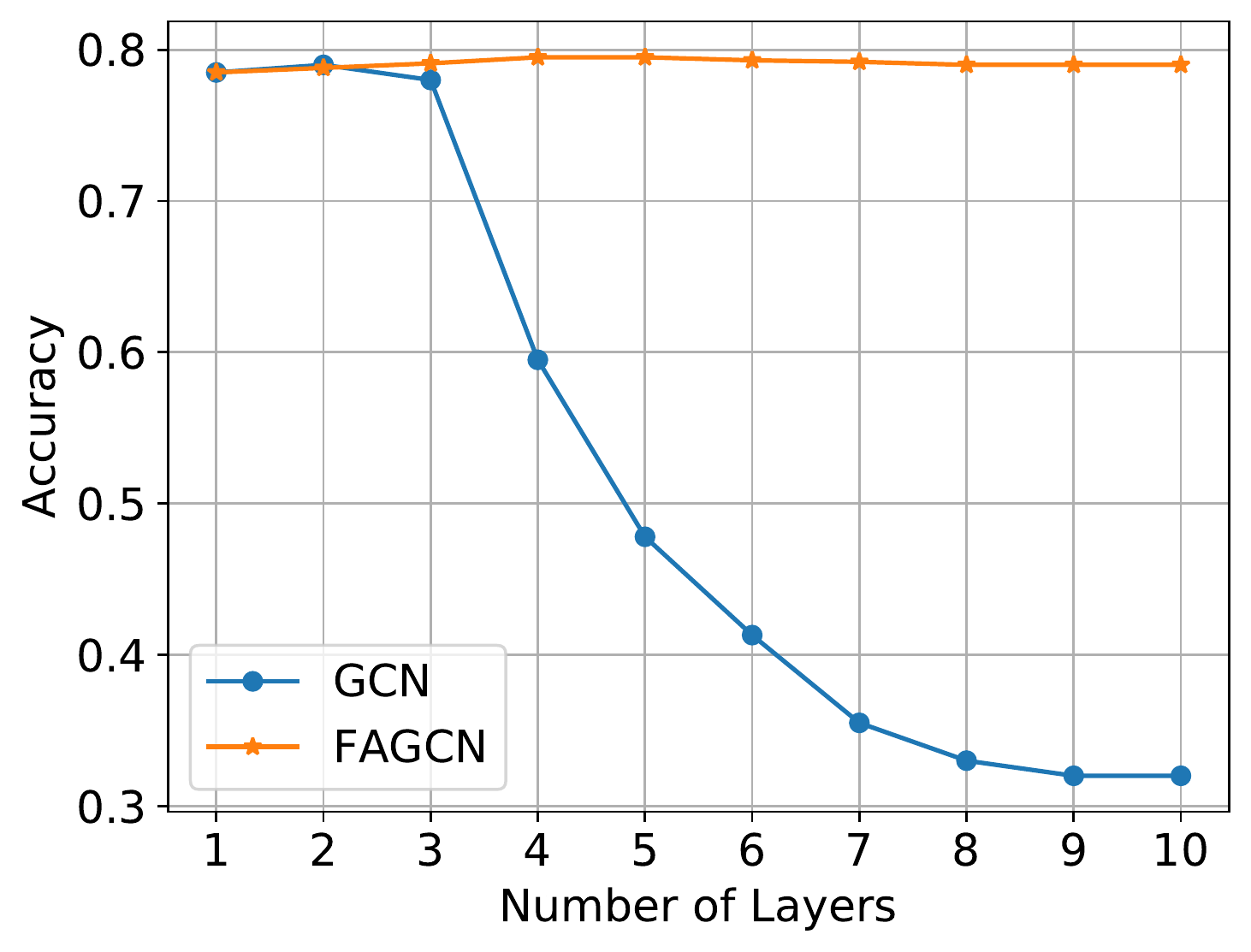}
    }
	\subfigure[Chameleon]{
        \includegraphics[width=0.45\linewidth]{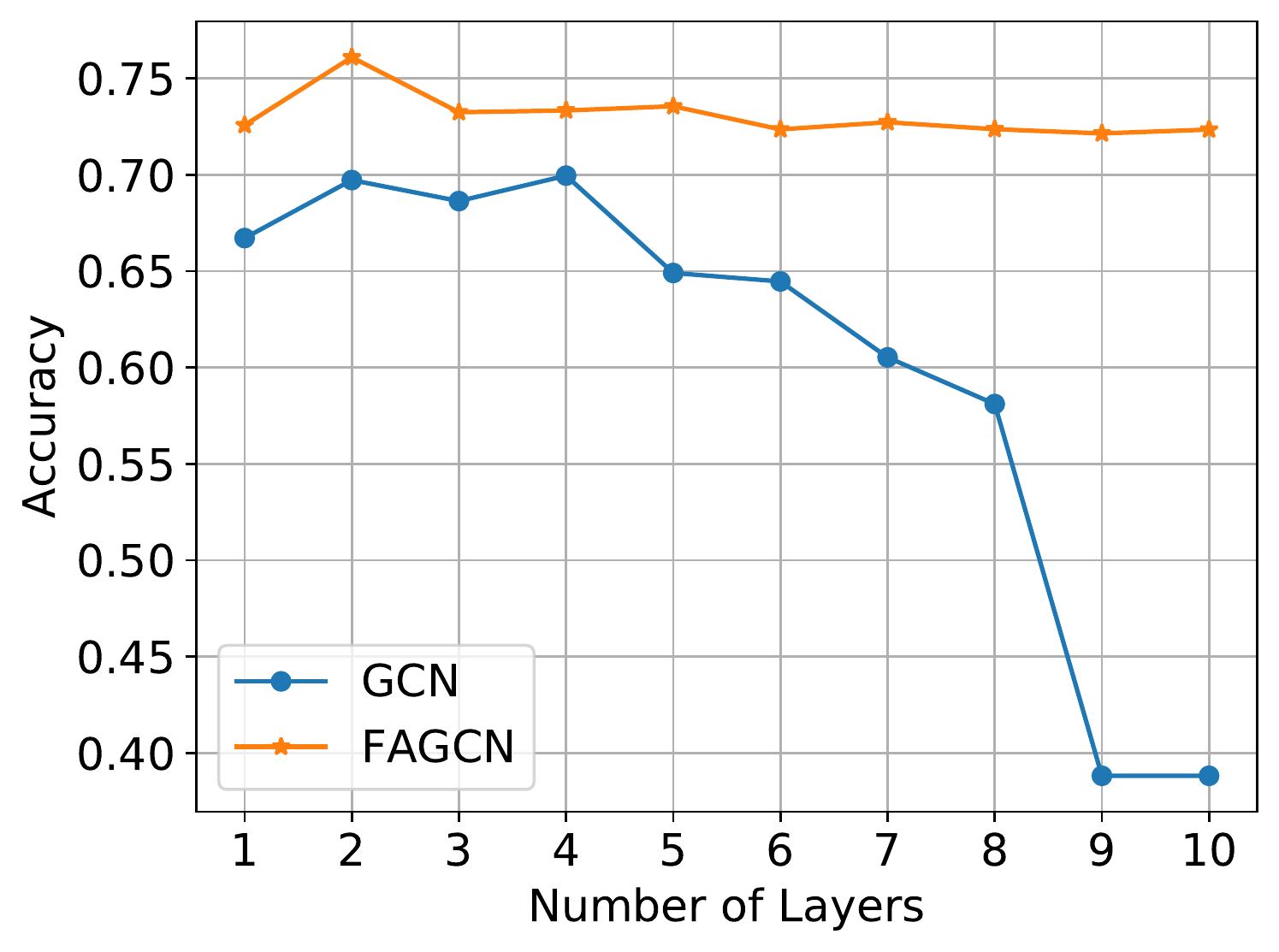}
    }
    \quad    
	\subfigure[Squirrel]{
        \includegraphics[width=0.45\linewidth]{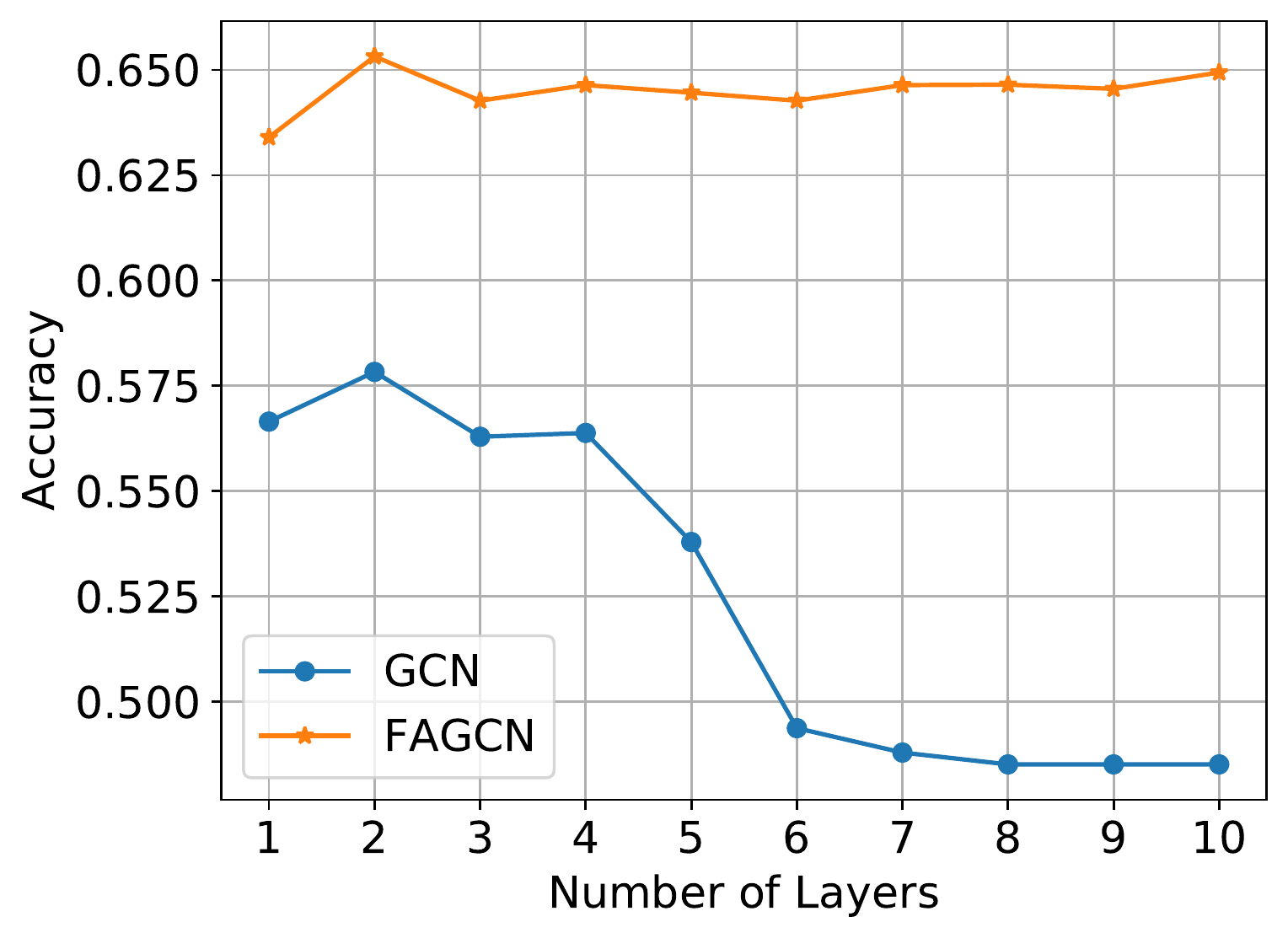}
    }
	\subfigure[Actor]{
        \includegraphics[width=0.45\linewidth]{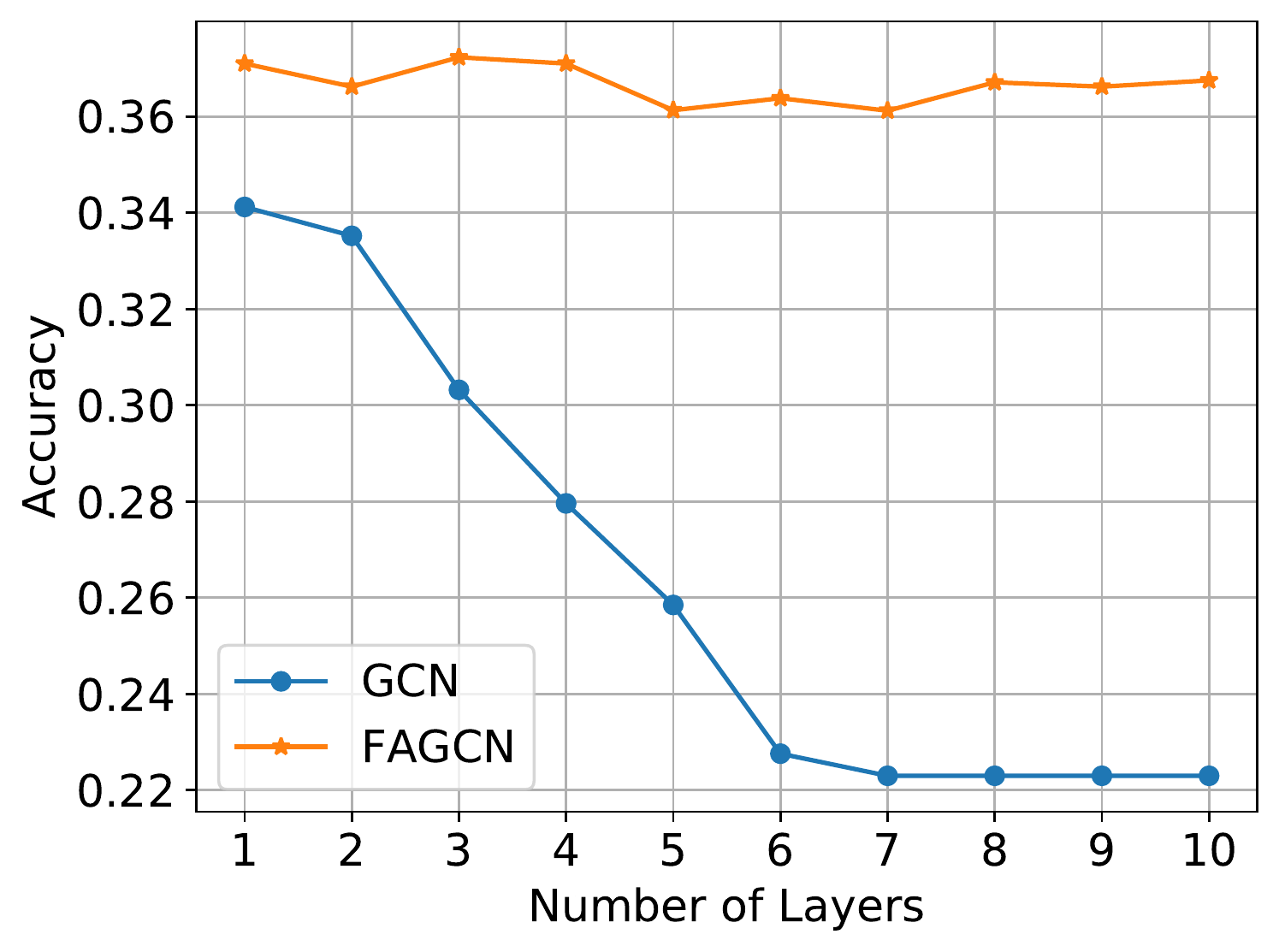}
    }
    \caption{Classification accuracy with different model depth.}
    \label{over-smoothing}
\end{figure}

\begin{figure*}[t]
\centering
\subfigure[Cora, Citeseer and Pubmed]{
\label{cora_h}
\includegraphics[width=0.23\textwidth]{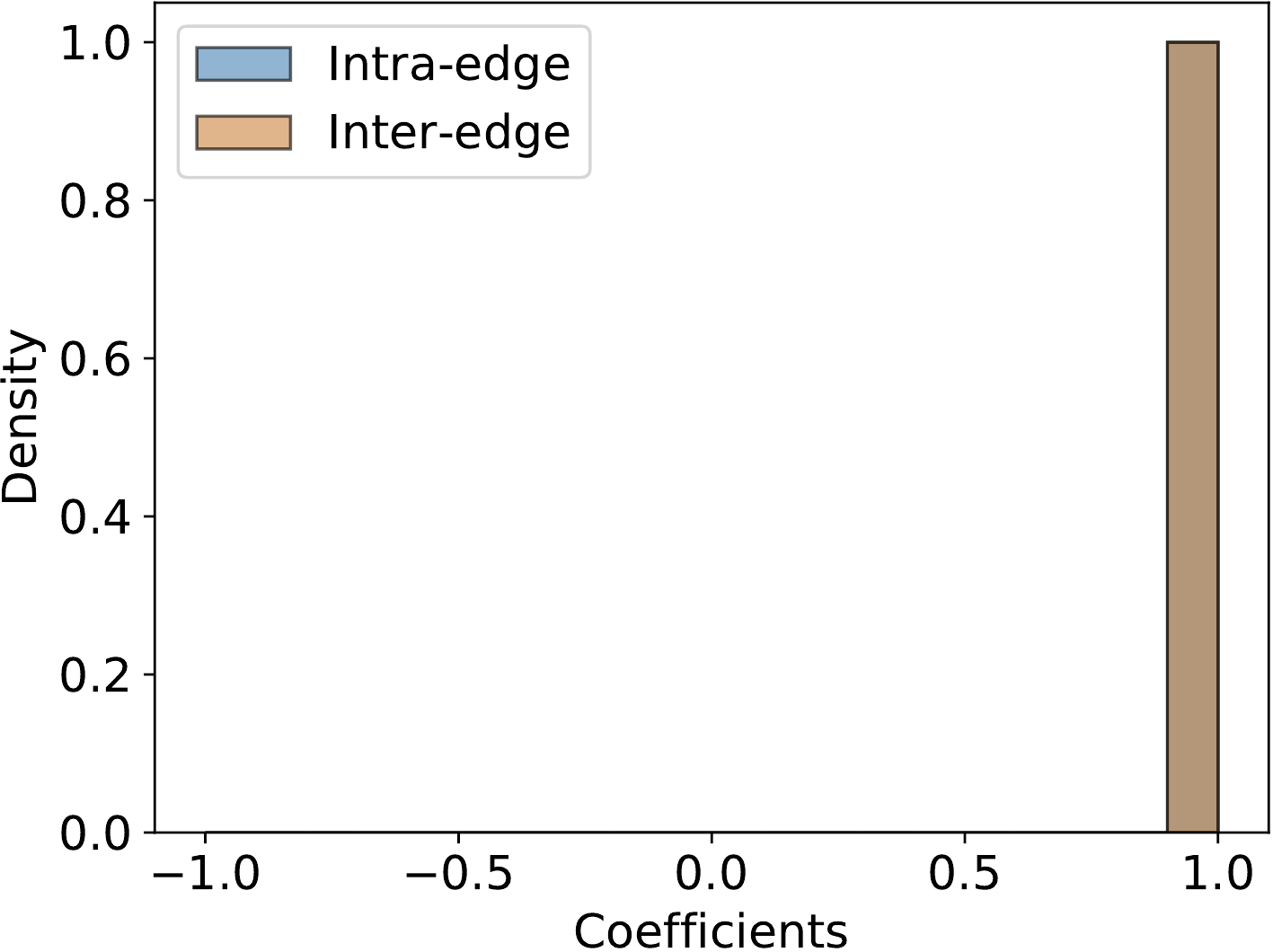}
}
\subfigure[Chameleon]{
\label{chameleon_h}
\includegraphics[width=0.23\textwidth]{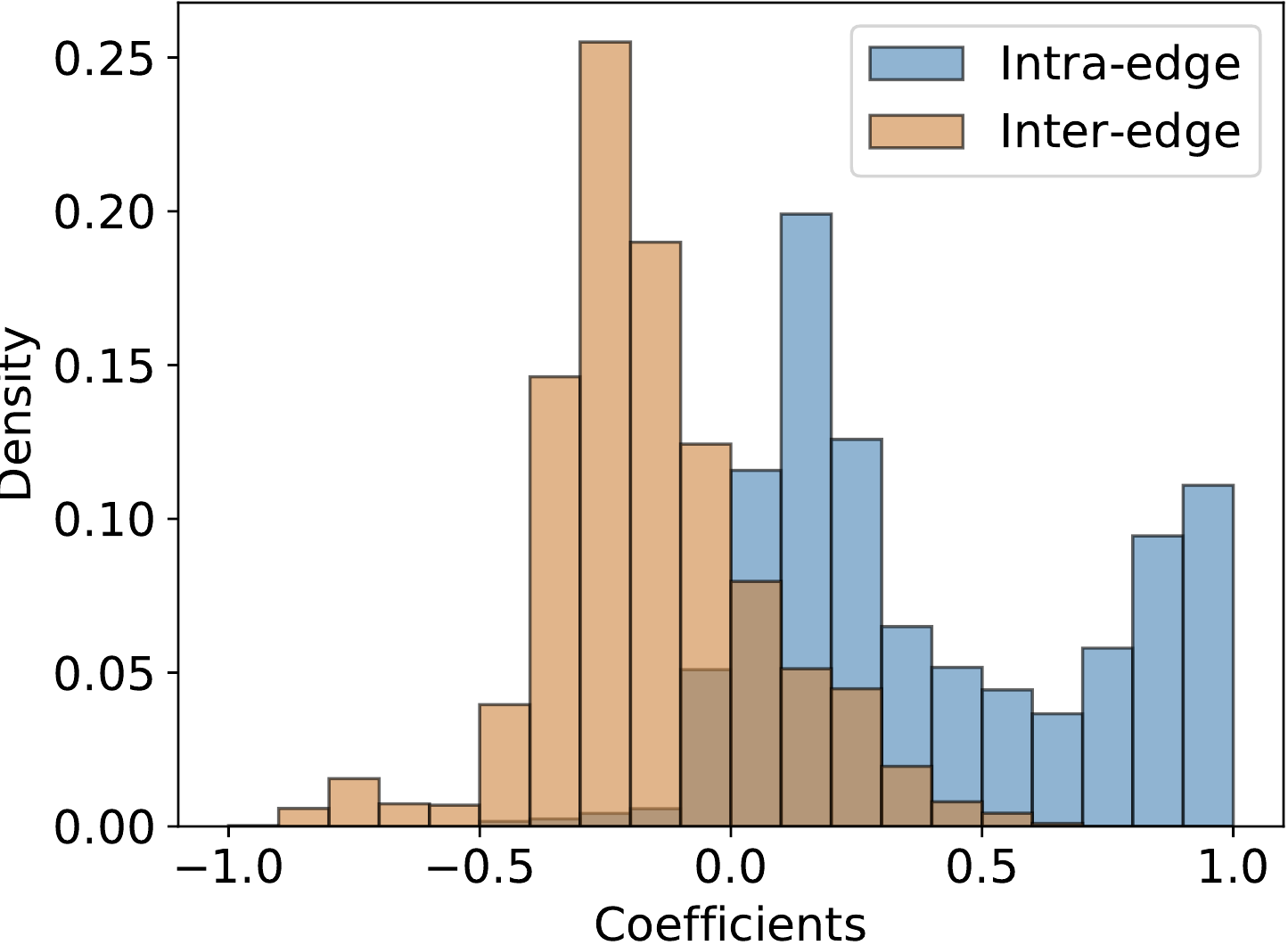}
}
\subfigure[Squirrel]{
\label{squirrel_h}
\includegraphics[width=0.23\textwidth]{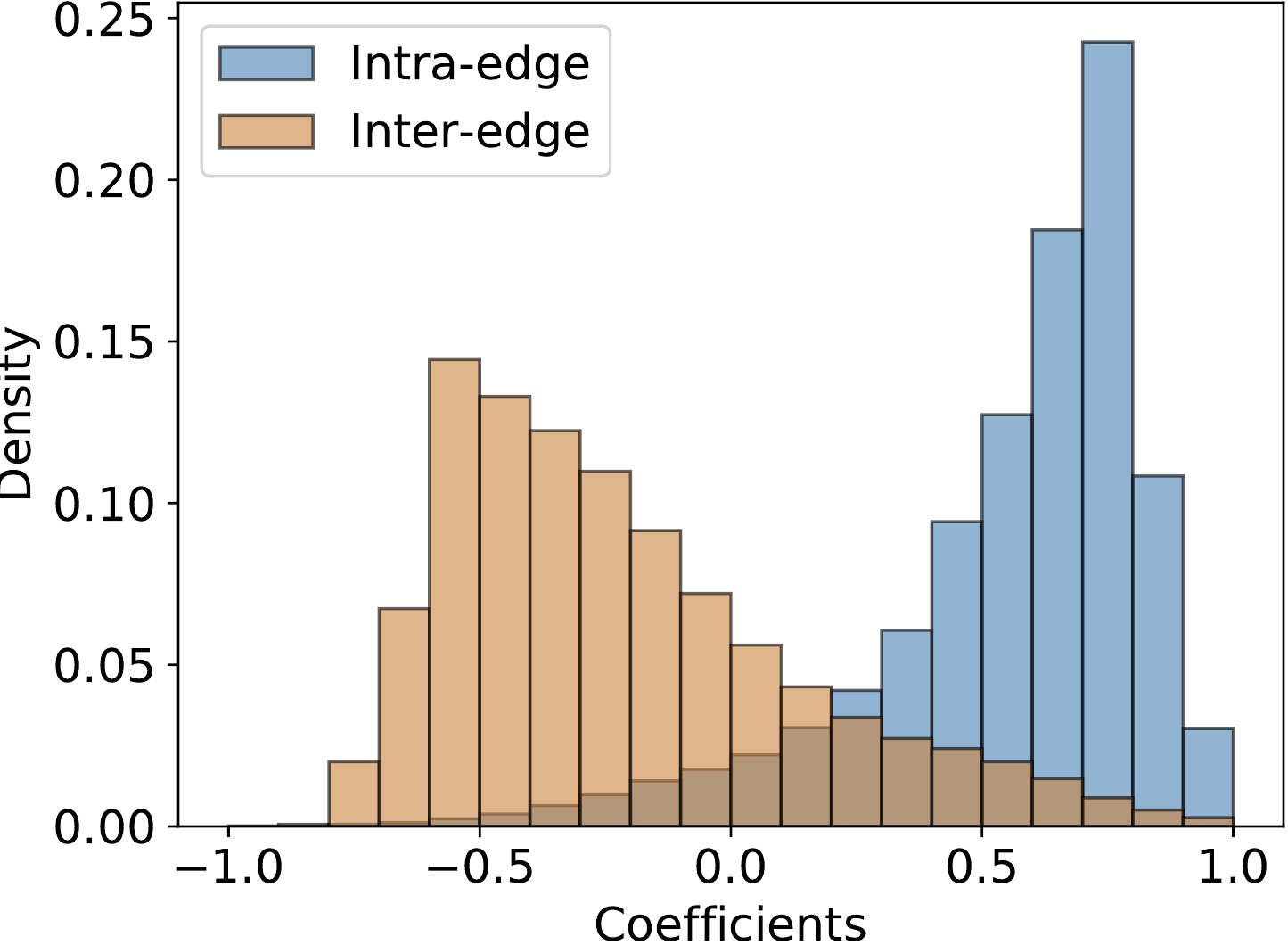}
}
\subfigure[Actor]{
\label{film_h}
\includegraphics[width=0.23\textwidth]{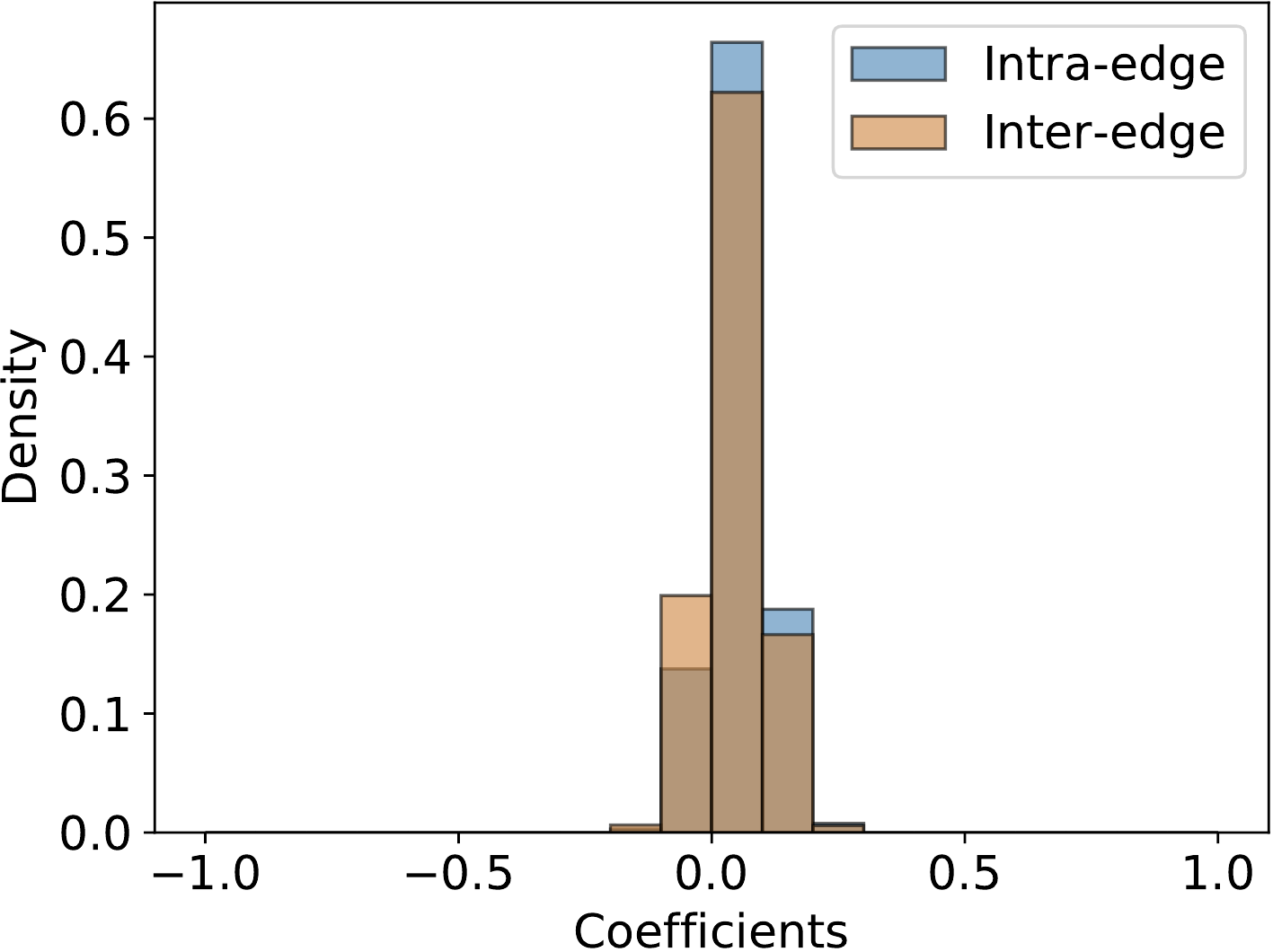}
}
\caption{Visualization of edge coefficients on different networks.}
\label{hist}
\end{figure*}

The performance of different methods on assortative networks is summarized in Table \ref{assortative}. GraphHeat designs a low-pass filer through heat kernel, which can better capture the low-frequency information than GCN \cite{GraphHeat}. Therefore, it performs best in the baselines. But we can see that FAGCN exceed the benchmarks on most networks due to the enhanced low-pass filter, which validates the importance of low-pass filters in the assortative networks. 

Besides, the performance on disassortative networks is illustrated in Fig. \ref{disassortative}. Note that we do not choose all baselines because the methods focus on low-pass filtering have poor performance, and we use GCN and GAT as representatives. In addition, APPNP leverages residual connection to preserve the information of raw features, ChebNet uses ChebNet polynomials to approximate arbitrary convolution kernels and Geom-GCN is the state-of-the-art on disassortative networks. Therefore, comparing FAGCN with these baselines can reflect the superiority of FAGCN. From Fig. \ref{disassortative}, we can see that GCN and GAT perform worse than other methods, which indicates that only using low-pass filters is not suitable for disassortative networks. APPNP and ChebNet perform better than GCN and GAT, which shows that the raw features and polynomials can preserve the high-frequency information to some extent. Finally, FAGCN performs best in most datasets and label rates, which reflects the superiority of our method.

\subsection{Alleviating Over-smoothing Problem}
To validate whether FAGCN can alleviate the over-smoothing problem, we compare the performance of GCN and FAGCN under different model depth. The results are shown in Fig. \ref{over-smoothing}. It can be seen that GCN achieves the best performance at two layers. As the number of layers increases, the performance of GCN drops rapidly, which indicates that GCN suffers from over-smoothing seriously. Instead, the results of FAGCN are stable and significantly higher than GCN on different types of networks. The reasons are two-folds: one is that in Section \ref{sec:expressive} we show that negative weights can prevents node representations from being too similar, which benefits deeper network architecture. Another is that we add the raw features, containing both low-frequency and high-frequency information, to each layer, which further keeps node representations from becoming indistinguishable. Through these two designs, when the model going deep, the performance of FAGCN is significantly better than GCN, which indicates that FAGCN has a good capability to alleviate over-smoothing.



\subsection{Visualization of Edge Coefficients}
\label{visualexperiment}

In order to verify whether FAGCN can learn different edge coefficients to adapt to different networks, we visualize the coefficients $\alpha_{ij}^{G}$, extracted from the last layer of FAGCN.
Specifically, we divide the edges into intra-edges and inter-edges based on whether two connected nodes have the same label. 
It can be seen from Fig. \ref{cora_h} that in the networks with large assortativity, i.e., Cora, Citeseer and Pubmed, all edges are concentrated at the positive weights, which implies that the low-pass filter plays a major role in classification.
However, in Fig. \ref{chameleon_h} and \ref{squirrel_h}, a lot of inter-edges are distributed in negative weights, which shows that in the network with small assortativity, the high-frequency signal plays an important role in node classification.
Moreover, there is an interesting phenomenon that in Fig. \ref{film_h} the coefficients of edges are concentrated at zero. One possible reason is that the assortativity of Actor is quite small, which implies that the structures contributes less to the results of node classification, instead, the raw features dominate the classification results.


\subsection{Details of Wikipedia Networks}

In this section, we aim to give more details of Wikipedia networks. First of all, Chameleon and Squirrel were originally collected for regression task, i.e., traffic prediction \cite{wikinet}. We divide the traffic into three categories: \emph{less than 1000}, \emph{between 1000 and 10000} and \emph{more than 10000}, so that they can be applied to node classification task. Secondly, labels of Chameleon and Squirrel are different from those in Geom-GCN. The reason is that in the disassortative networks provided by \cite{GeomGCN}, i.e., Cham-5 and Squi-5 in Table \ref{tab:compare}, we find that GCN performs much better than MLP. This is a strange phenomenon, because MLP uses raw features as input, which contains high-frequency information, so its performance should be better than GCN \cite{H2GNN}. Therefore, we redivide the labels based on traffic, i.e., Cham-3 and Squi-3 in Table \ref{tab:compare}, where the performance of GCN and MLP is more reasonable. Besides, we can see that FAGCN perform best on all four datasets, so its effectiveness is still guaranteed across different datasets.


\begin{table}[t]
  \centering
  \caption{Classification accuracy with different label division.}
    \begin{tabular}{lrrrr}
    \toprule
    \multirow{2}[4]{*}{\textbf{Method}} & \multicolumn{2}{c}{This paper} & \multicolumn{2}{c}{\cite{GeomGCN}} \\
\cmidrule(lr){2-3} \cmidrule(lr){4-5}          & \multicolumn{1}{c}{\textbf{Cham-3}} & \multicolumn{1}{c}{\textbf{Squi-3}} & \multicolumn{1}{c}{\textbf{Cham-5}} & \multicolumn{1}{c}{\textbf{Squi-5}} \\
    \midrule
    FAGCN		& \textbf{76.1\%} & \textbf{66.7\%} & \textbf{61.7\%} & \textbf{39.7\%} \\
    Geom-GCN	& 73.2\% & 63.3\% & 60.9\% & 38.1\% \\
    \midrule
    GCN   		& 72.3\% & 61.9\% & \textbf{59.8\%} & \textbf{36.9\%} \\
    MLP   		& \textbf{74.3\%} & \textbf{63.1\%} & 46.4\% & 29.7\% \\
    \bottomrule
    \end{tabular}
  \label{tab:compare}
\end{table}

\section{Related Work}


\textbf{Spectral Graph Neural Networks.} Spectral GNNs aim to define the convolution kernel in spectral domain, by leveraging the theory of graph signal processing. Spectral CNN \cite{SpectralCNN} treats the convolution kernel as a trainable diagonal matrix and directly learns the amplitudes of signals. However, it requires the decomposition of Laplacian matrix, which is inefficient. To deal with this problem, ChebNet \cite{ChebNet} uses the polynomial of Laplacian matrix to approximate the convolution kernel and make better performance. GCN \cite{GCN} is the first-order approximation of ChebNet with a self-loop mechanism. GraphHeat\cite{GraphHeat} designs a more powerful low-pass filter through heat kernel. Besides, GWNN \cite{GWNN} replaces the eigenvectors with wavelet bases so as to further improve the efficiency of the model. Generally, spectral methods have good interpretability for the signal processing on graph, but lack generalization \cite{GraphSAGE}.

\textbf{Spatial Graph Neural Networks.} Spatial GNNs focus on the design of aggregation function. GraphSAGE \cite{GraphSAGE} designs a permutation-invariant aggregator for message passing; GAT \cite{GAT} employs self-attention to calculate the coefficients of neighbors in aggregation; MoNet \cite{MoNet} provides a unified generalization of graph convolutional architectures in spatial domain; PPNP \cite{PPNP} incorporates personalized PageRank to the aggregation function;
Geom-GCN \cite{GeomGCN} utilizes the structural similarity to capture the long-range dependencies in disassortative graphs.
H2GNN \cite{H2GNN} separates the raw features and aggregated features so as to preserve both high-frequency and low-frequency information, but it lacks adaptability;
Non-Local GNN \cite{Non-Local} designs an attention-guided sorting mechanism to transform the disassortative networks into assortative networks, which costs a lot of computations.
Generally, spatial methods are more flexible and scalable, but lack interpretability.
It is worth noting that FAGCN is a spatial method, but it still has good interpretability, which combines the advantages of both spectral and spatial methods.

\section{Conclusion}

In this paper, we make the attempt to study the roles of low-frequency and high-frequency signals in GNNs and show that both of them are helpful in learning node representations. Based on this observation, we design a novel frequency adaptation graph convolutional network to adaptively combine the low-frequency and high-frequency signals. Theoretical analysis shows that the expressive power of our model is greater than most existing GNNs.
An important direction of future work is to use more signals with different frequencies, e.g., the intermediate frequency signals.

\section{Acknowledgments}
This work is supported in part by the National Natural Science Foundation of China (No. U1936220, U1936104, 61772082, 61702296, 62002029, 61972442), Meituan-Dianping Group and BUPT Excellent Ph.D. Students Foundation (No. CX2020115). Huawei Shen is funded by Beijing Academy of Artificial Intelligence (BAAI).

\bibliography{FAGCN}

\end{document}